\documentclass[twoside]{article}

\usepackage[accepted]{aistats2022}

\usepackage[round]{natbib}

\usepackage[utf8]{inputenc} %
\usepackage[T1]{fontenc}    %
\usepackage{hyperref}       %
\usepackage{url}            %
\usepackage{booktabs}       %
\usepackage{amsfonts}       %
\usepackage{nicefrac}       %
\usepackage{microtype}      %
\usepackage{xcolor}         %
\usepackage{caption}
\usepackage{subcaption}
\usepackage{multirow}
\usepackage{amsmath}

\usepackage[ruled, lined, linesnumbered, commentsnumbered, longend]{algorithm2e}

\usepackage{todonotes}

\usepackage{amsthm,amssymb}
\usepackage{amsmath}
\usepackage{amsfonts}
\usepackage{bbm}
\usepackage{xspace}
\usepackage{tikz}
\usepackage{enumitem} %
\usepackage{graphbox} %
\usepackage[normalem]{ulem}
\usepackage{wrapfig}
\usepackage[export]{adjustbox} %
\usepackage[noend]{algorithmic}
\usepackage{float}
\usepackage{hyperref}
\usetikzlibrary{decorations.pathreplacing}

\newif\ifcomment
\commenttrue

\newcommand{\kconv}[1]{\ensuremath{\sup_u \{
        l(\var,u) - y\cdot \left(1 - k(u,{#1})\right)
        \}
        }\xspace}
\newcommand{\kconvmult}[1]{\ensuremath{\sup_u \{
        l(\var,u) k(u,{#1})
        \}
        }\xspace}
\newcommand{\fsig}{\ensuremath{l^k}\xspace}

\newcommand{\phenv}{Pasch-Hausdorff envelope\xspace}
\newcommand{\menv}{Moreau envelope\xspace}

\newcommand{\gkrnl}{Gaussian RBF kernel\xspace}
\newcommand{\ckrnl}{$c$-exponential kernel\xspace}
\newcommand{\lapkrnl}{Laplacian kernel\xspace}

\newcommand{\mreg}{Moreau-Yosida regularization\xspace}
\newcommand{\kop}[1]{\mathcal{T}_{#1}} %

\newcommand{\kdro}{Kernel DRO\xspace} %
\newcommand{\wdro}{Wasserstein DRO\xspace} %
\newcommand{\wsd}{Wasserstein distance\xspace} %
\newcommand{\idro}{KI\xspace} %
\newcommand{\sdro}{ARKS\xspace} %
\newcommand{\cdro}{KE\xspace} %

\newcommand{\ctran}{$c$-transform\xspace}
\newcommand{\ktran}{$k$-transform\xspace}

\newcommand{\eqWRM}{\eqref{eq:sinha}\xspace}

\newcommand{\var}{\ensuremath{\theta}\xspace} %
\newcommand{\pemp}{\ensuremath{\hat{P}_N}\xspace} %
\newcommand{\ptrue}{\ensuremath{P_0}\xspace} %

\newcommand{\hnorm}[1]{\|{#1}\|_{\mathcal H}}
\newcommand{\lnorm}[1]{\operatorname{Lip}({#1})} %
\newcommand{\expected}[2]{{\mathbb E}_{#1}\ #2}
\newcommand{\rkhs}{\ensuremath{\mathcal H}\xspace}

\newcommand{\domain}{\ensuremath{\mathcal X}\xspace}
\newcommand{\ipm}{\ensuremath{\gamma_{\mathcal F}}\xspace} %

\DeclareMathOperator*{\sjt}{\mathrm{subject}\ \mathrm{to}}

\newtheorem{proposition}{Proposition}[section]

\newtheorem{lemma}[proposition]{Lemma}

\newtheorem{definition}{Definition}

\newtheorem*{remark}{Remark}

\makeatletter
\newcommand\footnoteref[1]{\protected@xdef\@thefnmark{\ref{#1}}\@footnotemark}
\makeatother

\usepackage{xcolor}

\ifcomment

\newcommand{\jz}[1]{
    [\textcolor{red!80!black}{\textbf{JZ:}} \textcolor{red!80!black}{\footnotesize\textbf{#1}}]
}

\else
\newcommand{\letsremove}[1]{}
\newcommand{\jz}[1]{}
\fi

\newcommand{\figtoyA}{Figure~\ref{fig:sgdstep_landscape}}

\newcommand{\eqKDRO}{\eqref{eq:kdro_dual}\xspace}

\usepackage{cleveref}
\crefformat{footnote}{#2\footnotemark[#1]#3}

\begin{document}

\runningauthor{Zhu, Kouridi, Nemmour, Sch\"olkopf}

\twocolumn[

\aistatstitle{Adversarially Robust Kernel Smoothing}

\aistatsauthor{ Jia-Jie Zhu
\\ Empirical Inference Department\\ Max Planck Institute for Intelligent Systems\\ T\"ubingen, Germany\\
and Weierstrass Institute\\
Berlin, Germany\\
\texttt{zhu@wias-berlin.de}
\And Christina Kouridi\\ Empirical Inference Department\\ Max Planck Institute for Intelligent Systems\\ T\"ubingen, Germany
\\
Currently at InstaDeep Ltd.\\
London, United Kingdom
\\
\texttt{christinakouridi@gmail.com}
\AND Yassine Nemmour \\ Empirical Inference Department\\ Max Planck Institute for Intelligent Systems\\ T\"ubingen, Germany \\ \texttt{ynemmour@tuebingen.mpg.de}
\And Bernhard Sch\"olkopf\\ Empirical Inference Department\\ Max Planck Institute for Intelligent Systems\\ T\"ubingen, Germany \\ \texttt{bs@tuebingen.mpg.de}}

\aistatsaddress{  }
]

\begin{abstract}
We propose a scalable robust learning algorithm combining kernel smoothing and robust optimization. Our method is motivated by the convex analysis perspective of distributionally robust optimization based on probability metrics, such as the Wasserstein distance and the maximum mean discrepancy. We adapt the integral operator using supremal convolution in convex analysis to form a novel function majorant used for enforcing robustness. Our method is simple in form and applies to general loss functions and machine learning models.
Exploiting a connection with optimal transport, we prove theoretical guarantees for certified robustness under distribution shift.
Furthermore, we report experiments with general machine learning models, such as deep neural networks, to demonstrate competitive performance with the state-of-the-art certifiable robust learning algorithms based on the Wasserstein distance.
\end{abstract}

\section{Introduction}
When learning with finitely many samples, there is an inevitable distribution shift between the training data and the test data, characterized by empirical process theory \citep{vaartWeakConvergenceEmpirical2013}.
A lack of causal inference can also cause learners to lose robustness under the shifted distribution \citep{meinshausenCausalityDistributionalRobustness2018}.
Furthermore, potential malicious \emph{adversaries} may create large artificial distribution shifts to hamper modern deep learners \citep{madryDeepLearningModels2019}.
Hence, learning under \emph{distribution shift} presents significant challenges to current machine learning algorithms.

\emph{Distributionally robust optimization} (DRO) \citep{delageDistributionallyRobustOptimization2010,scarfMinmaxSolutionInventory1958} seeks to robustify against unknown distribution shift explicitly.
Given a loss function of interest $l(\var, \cdot)$,
it solves a \emph{robust optimization} \citep{soysterTechnicalNoteConvex1973,ben-talRobustOptimization2009a} problem
\newcommand{\introDro}{\eqref{eq:dro_intro}}
\begin{equation}
\label{eq:dro_intro}
\min_\var \sup_{P\in \mathcal C}\expected{\xi\sim P} l(\var, \xi), 
\end{equation}
where \var is the decision variable, $\xi$ noise or randomness, $\mathcal C$ a set of distributions over the uncertain variable $\xi$ that the optimizer wishes to robustify against, often referred to as the \emph{ambiguity set}.
DRO is particularly relevant to statistical machine learning as one may construct the ambiguity set $\mathcal C$ as a metric ball centering at the empirical distribution $\pemp$ that also contains the true data-generating distribution $\ptrue$.
For example, the performance guarantees for Wasserstein DRO have been established by \cite{mohajerinesfahaniDatadrivenDistributionallyRobust2018,zhaoDatadrivenRiskaverseStochastic2018}.
Here, the idea is to use known convergence rate results for empirical estimations of the underlying probability metrics, e.g., the Wasserstein metrics \citep{kantorovichSpaceTotallyAdditive1958} and the closely related \emph{integral probability metrics} (IPM) \citep{sriperumbudurEmpiricalEstimationIntegral2012,mullerIntegralProbabilityMetrics1997}.
Such metrics (or topologies) often correspond to smooth functions as their dual spaces, which characterize the empirical distribution's convergence to the true data-generating distribution \citep{vaartWeakConvergenceEmpirical2013,billingsleyWeakConvergenceMeasures1971}.
Researchers have proposed DRO algorithms with various statistically meaningful ambiguity sets in a large body of literature.
While not the focus of previous works, convex analysis tools
play important roles in enforcing distributional robustness. They are the primary tools we employ in this paper.

Certain simple DRO problems, such as linear classification with logistic regression losses, admit the tractable reformulation into convex problems, as studied by \cite{ben-talRobustSolutionsOptimization2013,namkoongVariancebasedRegularizationConvex,mohajerinesfahaniDatadrivenDistributionallyRobust2018,blanchetOptimalTransportBased2018,shafieezadeh-abadehDistributionallyRobustLogistic2015}.
However, this only applies to a limited class of convex loss functions and simple models, as also noted by, e.g., \cite{sinhaCertifyingDistributionalRobustness2017}.
For general machine learning models, e.g., deep neural networks (DNNs), and common losses $l$ in \eqref{eq:dro_intro}, there exists no tractable reformulation to solve DRO~\introDro.
This paper addresses general losses in machine learning tasks, which are less explored in terms of principled distributional robustness, save very few exceptions such as \cite{sinhaCertifyingDistributionalRobustness2017,blanchetOptimalTransportBased2018,zhuKernelDistributionallyRobust2020d}.

\newcommand{\wc}{worst-case\xspace}
\paragraph{Contribution.}
This paper 
leverages the critical roles smoothness and function majorants play in enforcing distributional robustness.
We summarize our contributions and sketch the main results.
\begin{enumerate}[noitemsep,topsep=0pt]
    \item We analyze the smooth function majorant perspective of distributional robustness, which generalizes the existing practice of using the \mreg in \wdro to flexibly chosen general majorants surrogate losses.
    Specifically, we propose the \ktran (Definition~\ref{def:ktran}) that adapts the convolution in the integral operator to the supremal convolution in convex analysis to form a new function majorant.
    \item 
    Using those tools, we propose a novel robust learning algorithm (Section~\ref{sec:algs}), 
    the \emph{adversarially robust kernel smoothing} (\sdro). It solves the minimax program
    \begin{equation}
            \min_\var
            \frac{1}N\sum_{i=1}^N 
            \Bigg\{
                \kconvmult{\xi_i}
            \Bigg\}.
    \end{equation}
    \item Exploiting a connection between the \ktran and optimal transport (OT), we provide theoretical guarantees in terms of robustness certificate for \sdro under distribution shift.
    Highlighting the role of kernel bandwidth, our analysis unifies the two perspectives of DRO using OT and kernel methods.
    \item 
    While \sdro is derived from kernel methods, it can be easily applied to large-scale machine learning with DNNs. For example, we report an experiment with a ResNet-20 model, where applying exact \wdro reformulation techniques (e.g.,~\citep{mohajerinesfahaniDatadrivenDistributionallyRobust2018,shafieezadeh-abadehDistributionallyRobustLogistic2015}) is out of the question.
    There, \sdro performs at least competitively with the WRM algorithm \eqWRM proposed by \cite{sinhaCertifyingDistributionalRobustness2017}.
    \item
    Our code is publicly available online at \url{https://github.com/christinakouridi/arks}.
\end{enumerate}

\paragraph*{Notation.}
In this paper, we refer to the uniform data distribution $\pemp:=\frac1N\sum_{i=1}^N\delta_{\xi_i}$ as the empirical distribution. We use \ptrue to denote the (unknown) true data-generating distribution.
For the loss functions of interest $l(\var, \xi)$, we sometimes omit \var when there is no ambiguity.
\var denotes the decision variables, such as the weights of neural networks.
$\xi\in \domain$ denotes the random variable of interest, e.g., input data or features. Its samples are denoted by $\xi_i$.
For conciseness, we limit the discussion to compact \domain's.
We use \rkhs to denote a function space in the context, e.g., RKHS introduced in the next section. 
$\lnorm{}$ denotes the Lipschitz semi-norm.
$\hnorm{}$ is the RKHS norm in the context.
We will assume loss functions $l$ to be bounded continuous functions throughout the paper; extensions to upper semi-continuity in optimization settings is straightforward. See, e.g., \citep{shapiroLecturesStochasticProgramming2014}.
Variables are in their vectorial representation, e.g., $x=[x_1,\dots,x_N]$, and $f(x)=[f(x_1),\dots,f(x_N)]$.
Throughout the paper, we will refer to DRO using the Wasserstein distance as \wdro.
\newcommand{\lyphat}{\ensuremath{\widehat{l_{y, p}}}}
We refer to the \mreg
$
\lyphat(x) :=\sup_u\{
        l( u) - y\cdot \|u - x\|^p
\}
$
as
the \emph{supremal convolution} of a function $l$ and the (scaled) norm function $ y\|\cdot\|^p$.
To avoid ambiguity, we refer to the maximization of a concave function as a convex program.
Finally, $\gamma$ denotes some metric or divergence measure in the probability simplex.

\section{Reproducing Kernel Hilbert Spaces}
\label{sec:bg_apprx}
A learning task can be mathematically described as a function approximation problem
${\min_{f\in H}\|f - l\|_{\cdot}}$,
for some criterion $\|\|_{\cdot}$, e.g., function norm.
The target function
$l$ is often only known at certain data points $[x_1, ..., x_N]$.
One way to approach the function approximation problem is to consider a function approximator of the form
$
 \sum_{j=1}^N a_j k (x_i,x_j) = l(x_i), 1\leq i\leq N,
$
where $a_j$ are the coefficients to be determined and $k (x_i,x_j)$ some bi-variate function.
It is in our interest that the matrix $[k (x_i,x_j)]_{i,j}$ should be positive definite.
Motivated by this, we now define a symmetric real-valued function $k$ as a positive (semi-)definite kernel if $\sum_{i=1}^n \sum_{j=1}^n a_i a_j k(x_i, x_j)\ge 0$ for any $n \in \mathbb{N}$, $\{ x_i \}_{i=1}^n \subset \mathcal{X}$, and $\{a_i\}_{i=1}^n \subset \mathbb{R}$.
It is known \citep[e.g.][Chapter 2]{SchSmo02} that there is a one-to-one relationship between every positive semi-definite kernel $k$ and a Hilbert space \rkhs, whose feature map $\phi\colon
\mathcal{X} \to \rkhs$ satisfies $k(x,y) = \langle \phi(x), \phi(y)
\rangle_\rkhs$.
This Hilbert space is \emph{reproducing}, meaning that 
$f(x) = \langle f, \phi(x) \rangle_\rkhs$ for all $f\in \rkhs, x \in
\mathcal{X}$.
We call \rkhs the reproducing kernel Hilbert space (RKHS), also termed the native space of the kernel $k$.
RKHSs are widely used in function approximation based on data due to their attractive properties.

In addition to the functional approximation aspect, the RKHS has also been recently used to manipulate distributions, leveraging its statistical properties as the so-called Glivenko-Cantelli classes; cf. \citep{vaartWeakConvergenceEmpirical2013}.
Relevant to the robustness aspect,  the \emph{maximum mean discrepancy} (MMD, \citep{grettonKernelTwoSampleTest2012}) associated with an RKHS \rkhs is a metric in the probability simplex,
$
\gamma_\rkhs(P, Q )
 = \sup_{
\|f\|_\rkhs \le 1} \int f \, d(P-Q),
$
given the associated kernel is characteristic.
In particular, the minimax optimal rate for the MMD empirical estimation has been studied by \cite{tolstikhin2016minimax}, which can be used to set the ambiguity set level $\epsilon$ in the DRO problem~\eqref{eq:kdro_dual} with computable constants, independent of the dimensions.
In contrast, measure concentration rates (used for \wdro in \citep{mohajerinesfahaniDatadrivenDistributionallyRobust2018}) for the \wsd are dimension-dependent.
The MMD also has a closed-form estimator, while computing \wsd is hard in general~\citep{peyreComputationalOptimalTransport2019,santambrogioOptimalTransportApplied2015}.
MMD can be generalized to the integral probability metrics (IPM)~\citep{mullerIntegralProbabilityMetrics1997} defined by some function class $\mathcal F$, i.e.,
${\ipm (P,\hat{P}):=\sup_{f\in \mathcal F}\int f d (P-\hat{P})}.$
The well-known choices relevant to this paper include: $\mathcal F=\{f:\lnorm{f}\leq 1\}$ recovers the type-1 Wasserstein metric (Kantorovich metric);
the RKHS norm-ball $\mathcal F=\{f:\hnorm{f}\leq 1\}$ recovers the MMD.

Given a kernel $k$ and probability measure $\mu$, recall that the integral operator $\kop{}: L^2_\mu\to \rkhs$ is defined as
\begin{equation}
    \label{eq:int_op}
    \begin{aligned}
        \kop{}\ l (x)  := \int l(z) k(x,z) d\mu(z).
    \end{aligned}
\end{equation}
The integral operator maps $L^2_\mu$ to a subspace of the RKHS \citep{wendlandScatteredDataApproximation2004,conwayCourseFunctionalAnalysis2019} and
is used in the celebrated Mercer's theorem to characterize the eigendecomposition of RKHS functions.
In the context of this paper, we view the integral operator as a smoothing operation.

\section{Distributionally Robust Optimization for Machine Learning}
\label{sec:bg_dro}
We limit our discussion to DRO using the Wasserstein metrics~\citep{mohajerinesfahaniDatadrivenDistributionallyRobust2018,zhaoDatadrivenRiskaverseStochastic2018,gaoDistributionallyRobustStochastic2016,blanchetQuantifyingDistributionalModel2017}, the MMD~\citep{zhuKernelDistributionallyRobust2020d,staibDistributionallyRobustOptimization2019}, 
and general IPMs~\citep{zhuKernelDistributionallyRobust2020d}.
See \citep[Section~1.1]{gaoDistributionallyRobustStochastic2016,grettonKernelTwoSampleTest2012,arbelGradientRegularizersMMD2021,peyreComputationalOptimalTransport2019,weedSharpAsymptoticFinitesample2017} for the details of why those probability metrics are more advantageous in many machine learning applications than, e.g., $f$-divergences.
For convenience, we now restate the data-driven DRO primal formulation in \eqref{eq:dro_intro} with a probability discrepancy constraint.

\begin{equation}
    \label{eq:dro_metric}
    (\emph{DRO-\textrm{Primal}}):\ \min_\var \sup_{\gamma(P,\pemp)\leq\epsilon}\expected{P} l(\var, \xi),
\end{equation}
The discrepancy measure $\gamma$ in \eqref{eq:dro_metric} can be chosen to be  an IPM or OT metric. 

In general, solving the minimax DRO problem~\eqref{eq:dro_metric} requires a reformulation via the duality of conic linear optimization, cf. \citep{shapiroDualityTheoryConic2001,ben-talDerivingRobustCounterparts2015}.
While the \wsd has become the most popular choice for the DRO problem,
it is important to understand that one \emph{cannot} simply reformulate any \wdro problem as a convex program, except for very simple losses such as logistic regression~\citep{shafieezadeh-abadehDistributionallyRobustLogistic2015}.
Unfortunately,
for many practical machine learning models, there exists no exact tractable reformulation.
Popular \wdro approaches such as those proposed in \citep{mohajerinesfahaniDatadrivenDistributionallyRobust2018,zhaoDatadrivenRiskaverseStochastic2018} apply to a limited class of loss functions and models, such as logistic regression (linear classification).
Moreover, it is also known that estimating the Lipschitz constant for general models is intractable; cf. \citep{virmauxLipschitzRegularityDeep2018,biettiKernelPerspectiveRegularizing2019}, making Lipschitz regularization in \citep{shafieezadeh-abadehRegularizationMassTransportation2019} difficult.
This paper does not impose such restrictions on losses or models.
For commonly-used machine learning losses, it is well-known that one must resort to general approximate solution methods such as in~\citep{sinhaCertifyingDistributionalRobustness2017,blanchetOptimalTransportBased2018,zhuKernelDistributionallyRobust2020d}.

Most relevant to our work, the authors of \citep{sinhaCertifyingDistributionalRobustness2017} proposed to \emph{give up} certifying the exact distributional robustness level $\epsilon$ and apply a convexification technique using the \mreg, as approximate \wdro.
They solve the risk minimization problem, which they termed Wasserstein robust method (WRM),
\begin{multline}
    \label{eq:sinha}
        (\emph{WRM}):\
        \min_\var\frac{1}N\sum_{i=1}^N
        \Big\{
            \widehat{f^y_\var}(\xi_i):= \sup_u\{l(\var,u)\\
            - y\cdot c(u, \xi_i)\}
        \Big\},
\end{multline}
where $c$ is called the transport cost \citep{santambrogioOptimalTransportApplied2015}.
For example, when $c$ is the squared Euclidean distance, $\widehat{f^y_\var}(\xi_i) $ is referred to as the Moreau-Yosida regularization or \menv.
In that setting, WRM overcomes the hurdle of the aforementioned hardness of DRO for general machine learning tasks by virtue of a convexification effect. Intuitively, subtracting a strongly convex function makes the inner objective more concave. This technique was also used in
robust nonlinear optimization \citep{houskaNonlinearRobustOptimization2013}, trust-regions in numerical optimization (Chapter~4 of \citep{nocedal2006numerical}), and the S-procedure in robust control~\citep{polikSurveySLemma2007,yakubovichSprocedureNolinearControl1971}.
We refer interested readers to those works for detailed numerical procedures.
Later, we compare our novel kernel smoothing algorithm with WRM~\citep{sinhaCertifyingDistributionalRobustness2017} in experiments with general machine learning models, e.g., DNNs, to demonstrate our advantages over classical reformulation techniques.

On the other hand, if we choose the metric $\gamma$ to be an IPM associated with the function class $\mathcal F\subseteq \rkhs$, \cite{zhuKernelDistributionallyRobust2020d} proved the IPM-DRO duality.
It states that 
the primal DRO problem~\eqref{eq:dro_metric} is equivalent to solving the variational optimization problem
\begin{equation}
        \begin{aligned}
                \vspace{-0.2cm}
                (\emph{IPM-DRO}):\
                &\min_{ \var,f\in{\mathcal H}}& & \frac{1}N\sum_{i=1}^N f(\xi_i)+ \epsilon\hnorm{f}\\
                &\sjt& &
                l(\var, \xi)
                \leq
                f(\xi)   ,\ \forall \xi \in\domain a.e.
        \end{aligned}
        \label{eq:kdro_dual}
\end{equation}
Those authors also proposed approximate solution methods when the IPM is chosen as the MMD (see Section~\ref{sec:bg_apprx} for the advantages of MMD), which generalized the results of \cite{staibDistributionallyRobustOptimization2019} to general loss functions.
Through the lens of this paper, \eqKDRO explicitly seeks an upper \emph{envelope} $f$ of the loss $l$ as solutions to the variational dual program~\eqref{eq:kdro_dual}.
Instead of the Moreau-Yosida regularization, the smooth majorant role there is played by a more general smooth function $f\in\rkhs$.
Note that program~\eqref{eq:kdro_dual} is trivial if the loss $l$ is in an \rkhs and has a known RKHS norm.
The authors of \citep{zhuKernelDistributionallyRobust2020d} then proposed \kdro that makes it possible to use the MMD associated with any universal RKHSs for DRO and compute the rate for general losses.
To our knowledge, that is the only work aiming to exactly reformulate DRO for general machine learning models.
Compared to their method, we provide an approach that produces a function that satisfies the (semi-)infinite constraint in \eqref{eq:kdro_dual}, whereas \cite{zhuKernelDistributionallyRobust2020d}'s method can only satisfy that constraint approximately through constraint sampling.

To motivate our method, we make
two key observations into \eqWRM:
(1) the absence of the robustness level $\epsilon$ and (2) the fixed dual variable $y$.
That insight is also equivalent to
\emph{giving up the exact minimization} w.r.t. $f\in\rkhs$ and $\hnorm{f}$ in dual IPM-DRO~\eqref{eq:kdro_dual}, since fixing $y$ in \eqref{eq:sinha} is to not optimize w.r.t. the \menv $\widehat{f^y_\var}$.
This is equivalent to Lagrangian relaxation in nonlinear optimization.

\section{A Kernel Smoothing Algorithm for Robust Learning}
\label{sec:algs}
The so-called \ctran
$\widehat{f^y_\var}$ in WRM~\eqWRM, also known as the \mreg, plays a crucial role in the robustness of WRM~\eqWRM. Importantly, it is a \emph{majorant} function. 
\begin{definition}[Majorant]
    We say that $f$ is a majorant of $l$ if $f(\xi) \geq l(\xi)$ for $\xi$ a.e. in the domain of $l$.
\end{definition}
Notable examples of majorants relevant to DRO include the \mreg as well as the kernel functions in \eqKDRO.
In this paper, we also refer to a majorant as an upper envelope function.
To further make clear the roles that majorants play in robustness,
we provide a few convex analysis examples in the appendix regarding special cases of majorants relevant to DRO. 

Our \emph{key insight} from \eqKDRO, \eqWRM, and Lemma~\ref{thm:w1envelope} is that \emph{empirical risk minimization with the loss replaced by a majorant surrogate loss induces distributional robustness}.
Motivated by such relationship between robustness and the use of majorants, e.g., \ctran (and \mreg) in \eqWRM, kernel functions in \eqKDRO, we now introduce our robust learning algorithm.
\subsection{Adversarially Robust Kernel Smoothing}
\label{sec:smooth}
Our starting point is the minimax robust optimization (RO) problem~\citep{ben-talRobustOptimization2009a,soysterTechnicalNoteConvex1973}
\begin{equation}
	\label{eq:ro}
    (\emph{RO}):\
	\min_\var \sup_{u\in\domain} l(\var, u),
\end{equation}
where the learner assumes the uncertain variable $u$ to take the worst-case value.
It is easy to see the pessimism of RO as we do not know the true support of $u$ in machine learning.
On the other hand, empirical risk minimization (ERM; also referred to as the sample average approximation
 ) enjoys better performance but is more fragile to shift in distribution and uncertainty.
It can be seen as a simple form of smoothing by averaging.
Deviating from the typical DRO dual reformulation approaches, our \emph{key idea is to view smoothness as the opposite side of robustness} by manipulating the integral operator in RO.
To that end, let us first establish some tools.

The image of $l$ under the integral operator $\kop{}\ l \in\rkhs$ \eqref{eq:int_op} is a smooth function since it is in an RKHS, but does not directly enable robust learning.
To achieve robustness, we notice the following inequality
\begin{multline}
    \label{eq:sup_int}
    \kop{}\ l (x)  = \int l(z) k(x,z) d\mu(z) \\ \leq \sup_u \{
        l(u) k(u,{x})\},\ \forall x\in\domain,
\end{multline}
which is a straightforward inequality between expectation and supremum.
Using the right-hand-side expression, we propose the following majorant analogous to \ctran.
\begin{definition}[\ktran]
    \label{def:ktran}
    The \ktran of a function $l$ associated with kernel $k$ is defined as
    \begin{equation*}
        \label{eq:ktrans}
        {\fsig(x):=\sup_u \{l(u) k(u,{x})\}}.
    \end{equation*}
\end{definition}
It is helpful to think of concrete examples where $k(u,{x})$ is the \gkrnl or \lapkrnl.
To make our discussion more general, we now propose the following family of kernels inspired by the transport cost $c$ of OT.
\begin{definition}[\ckrnl]
    \label{def:ckernel}
    Suppose $c$ is the transport cost (as in the \wsd). The \ckrnl with bandwidth $\sigma>0$ is given by
    $
    k(x,x')=e^{-c(x,x')/\sigma}
    $.
\end{definition}
Note that a relevant kernel on probability metrics was studied in \citep{deplaenWassersteinExponentialKernels2020}.
For conciseness, we focus on the \gkrnl and \lapkrnl in the rest of this section.
Other constructions of majorants are possible and discussed in the appendix.
It is then straightforward to verify the following:
\begin{proposition}
    \label{thm:majorant}
    The \ktran of $l$ is a majorant of $l$. Furthermore, we have
        $\fsig \to l \textrm{ as } \sigma \to 0.$
\end{proposition}
Let us use the tools above to derive our robustification method.
We mitigate the conservatism of RO~\eqref{eq:ro} by replacing the the original loss $l$ with a \emph{smoothed } version
    $
	\min_\var \sup_{u}\kop{}\ l(\var,u)
    $
    .
In practice, we can only compute an empirical version of the integral operator based on data $\xi_i$. We have the following
\begin{multline}
	 \sup_{u}\Big\{\hat{\kop{}\ } l(\var,u) :=
	 	 \frac1N\sum_{i=1}^N k(\xi_i, u)l(\var,u) 
          \Big\} \\ \leq  \frac1N\sum_{i=1}^N  \kconvmult{\xi_i}.
          \label{eq:sup_int_op}
\end{multline}
The inner objective on the right-hand-side is the \ktran.
That objective is indeed less conservative than RO and more robust than ERM since
\begin{multline}
    \label{eq:erm_to_ro}
	\emph{(ERM)}:\ \sup_{u}\frac1N\sum_{i=1}^N l(\var,  \xi_i) \leq  \\
     \frac1N\sum_{i=1}^N  \kconvmult{\xi_i}
	\leq  \sup_ul(\var, u)\ :\emph{(RO)}
\end{multline}

for \ckrnl{}s, e.g., \gkrnl{}s.

We are now ready to propose the following novel robust learning scheme based on the insight from RO, \eqKDRO, and \eqWRM.
The \emph{main idea} here is simple: we minimize the risk using a surrogate loss constructed by the \ktran,
\begin{equation}
    \label{eq:ksmooth}
        (\emph{\sdro}):\
            \min_\var
            \frac{1}N\sum_{i=1}^N 
            \Big\{
                l^k_\var(\xi_i)\\
                := \kconvmult{\xi_i}
            \Big\}.
\end{equation}
\sdro resembles the risk minimization schemes using majorant surrogate losses in \eqKDRO and \eqWRM, but with our newly proposed \ktran.
Program~\eqref{eq:ksmooth} also bears a clear resemblance to the \emph{Nadaraya-Watson} model, and the \emph{vicinal risk minimization}~\citep{chapelleVicinalRiskMinimization} in the literature.
However, our approach differs in taking supremum to guarantee distributional robustness (see Section~\ref{sec:theory}) instead of merely smoothing.
Different from the existing robust kernel density estimation methods such as \citep{kimRobustKernelDensity2012}, which was applied by \cite{ningDatadrivenDecisionMaking2018} to learn uncertainty sets for robust optimization, \sdro considers specifically the worst-case risk of the loss $l$, rather than only performing general unsupervised density estimation.

Compared with existing DRO approaches, \sdro~\eqref{eq:ksmooth} does not use \emph{explicit regularization} (e.g., \citep{shafieezadeh-abadehRegularizationMassTransportation2019}), but an \emph{implicit} one.
To see that,
we establish the following.
\begin{proposition}
    \label{thm:recover_ro}
    Suppose the kernel bandwidth tends to infinity $\sigma\to\infty$, \sdro~\eqref{eq:ksmooth} is equivalent to the \emph{worst-case robust optimization} (RO)~\eqref{eq:ro}~\citep{ben-talRobustOptimization2009a,soysterTechnicalNoteConvex1973}.

    If kernel bandwidth tends to zero $\sigma\to 0$, then \sdro~\eqref{eq:ksmooth} recovers (ERM)
$
            \frac{1}N\sum_{i=1}^N 
            \min_\var l(\var, \xi_i).
$
\end{proposition}

If we choose a bandwidth $\sigma$ between those cases, the robustness is between RO and ERM, where DRO is.
\begin{remark}[Robustness, kernel bandwidth, and size of the function space]
Intuitively, if the kernel bandwidth is large, the function space becomes small.
In terms of robustness, our analysis is also consistent with the characterization of dual function space sizes for DRO and RO in \citep{zhuKernelDistributionallyRobust2020d}: large dual function spaces correspond to conservative but more robust optimization.
In contrast, smaller ones have better performance but are less robust.
We see those insights reflected in Proposition~\ref{thm:recover_ro} and further in Section~\ref{sec:theory}.
\end{remark}
Our risk minimization scheme \eqref{eq:ksmooth} can be straightforwardly used with stochastic gradient-based optimization for large-scale learning, e.g., with DNNs. We detail the training procedure in
Algorithm~\ref{alg:smooth}.
\begin{algorithm}[htb!]
  \caption{Robust Learning with \sdro}
  \label{alg:smooth}
  \begin{algorithmic}[1]
      \STATE {\bf input:} data sampler, initial iterate $\var_0$
      \FOR {$k=0,1,2,\dots,T$}
      \STATE sample $\{\xi_k\}$ and find 
      $u^*_k $ by maximizing $l(\var_k, u) k(u,{\xi_k})$ w.r.t. $u$
      \STATE update \var by stochastic gradient descent using estimate $\nabla_\var l(\var_k, u^*_k)$
      \ENDFOR
      \STATE {\bf output:} approximate solution $\var^*:=\var^T$
  \end{algorithmic}
  \end{algorithm}
Note that Step~3 of Algorithm~\ref{alg:smooth} can be seen as a proximal algorithm, discussed further in the next section.
\sdro~\eqref{eq:ksmooth} can also be interpreted as a form of \emph{adversarial training}~\citep{madryAdversarialRobustnessTheory,wongProvableDefensesAdversarial,goodfellowExplainingHarnessingAdversarial2015}:
for each $\xi_i$, the inner maximization problem of \eqref{eq:ksmooth} looks for an adversarial example $u$ that hurts the learner the most.
In the case of \gkrnl,
we show that the inner maximization objective in \eqref{eq:ksmooth} has favorable convexity structures for suitable choices of $\sigma$ in the next section, as well as in the appendix.

We further illustrate the geometric intuition of \sdro in \figtoyA{} using a toy problem.
For conciseness, we defer detailed experimental setups to the appendix.
The idea is to model the distribution shift and adversarial perturbation as a stochastic transition from the original state $X$ (illustrated as the black cross in \figtoyA{} (right)) to a perturbed uncertain state $U$ (the gray curve).
Concretely, the relationship in \eqref{eq:sup_int} (restated below for convenience) holds for any distribution $\forall \mu \in \mathcal P$,
\begin{multline*}
    \int l(z) k(x,z) d\mu(z) \leq \sup_u \{
        l(u) k(u,{x})\},\ \forall x\in\domain.
\end{multline*}
The left-hand-side (LHS) above can be interpreted as the \emph{conditional expectation}
        $
                \mathbb E [l(U)|X = x],
        $
which is the expected loss under the uncertain state $U$ after a distribution shift from the empirical data distribution.
In this context, a smooth kernel $k$ models a conditional density $P(U=z|X=x)$ of this stochastic transition.
Unlike $f$-divergence-based DRO~\citep{ben-talRobustSolutionsOptimization2013,namkoongStochasticGradientMethods2016}, our modeling does not require the shifted distribution to be absolute continuous (i.e., having the same support) w.r.t. the empirical distribution. \sdro uses the robust version of this operation characterized by the \ktran (i.e., right-hand-side (RHS) above) instead of the expectation on the LHS.
Therefore, the resulting surrogate worst-case loss (plotted in \figtoyA{} (left)) upper-bounds the original expected loss.
The worst-case adversarial perturbation is also plotted as a red cross (right).
See the caption of \figtoyA{} and the appendix for more details.
\begin{figure}[tb]
    \centering
    \begin{subfigure}{0.49\linewidth}
        \centering
        \includegraphics[height=1.25in,valign=c]{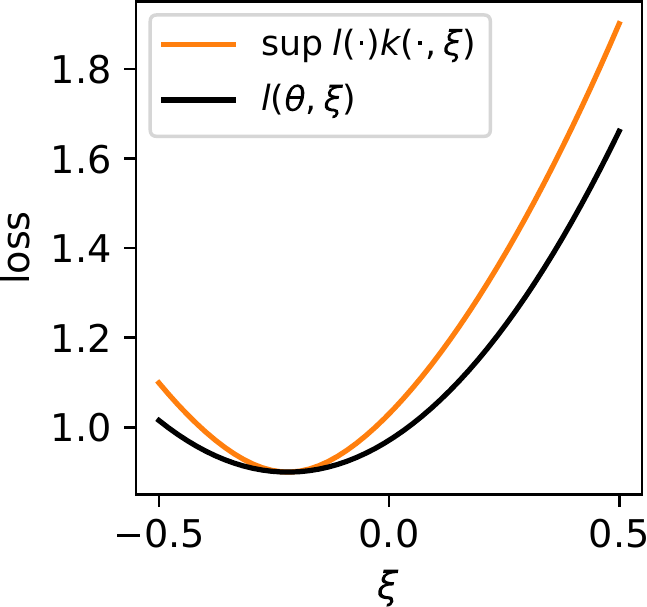}
        \label{fig:landscape}
    \end{subfigure}
    \begin{subfigure}{0.49\linewidth}
        \centering
        \includegraphics[height=1.35in,valign=c]{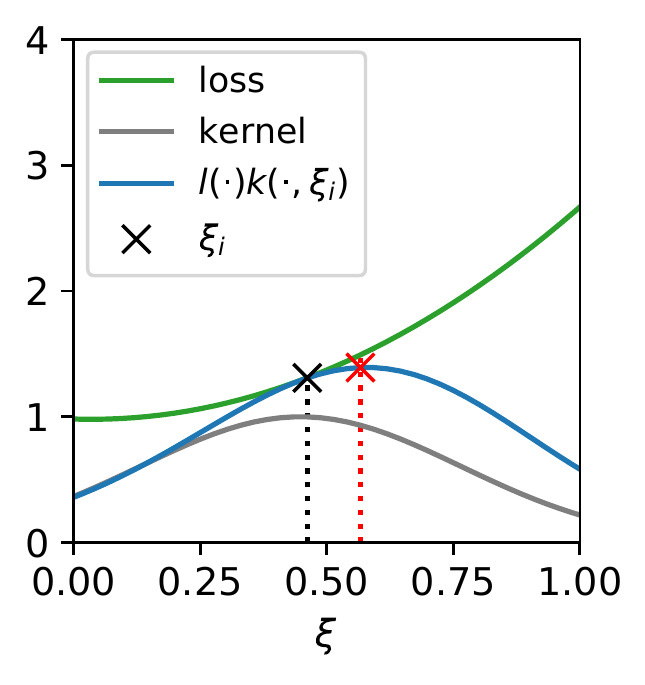}
        \label{fig:sgdstep}
    \end{subfigure}
    \caption{
    (\textbf{left})
        Loss landscape of the kernel robust smoothed loss
        $\fsig:=\sup_u \{
                l(u) k_{\sigma}(u,{\cdot})
                \}$. As analyzed in the text, as the width $\sigma$ decreases, the \sdro surrogate loss tends towards the original loss, i.e., $\fsig \to l \textrm{ as } \sigma \to 0$. Note that the kernel-smoothed loss $\fsig$ is a majorant of the original loss $l$.
    (\textbf{right}) 
        Illustration of the inner maximization problem of \sdro.
        This figure illustrates the mechanism that \sdro finds the adversarial example by kernel smoothing.
        The figure plots the original loss $l$ in green. The inner objective (using the \ktran in Definition~\ref{def:ktran}) is plotted in blue.
        The black cross is a sampled data point $\xi_i$.
        The red cross is the computed solution to the inner maximization problem of \sdro.
    }
    \vspace{-0.5cm}
    \label{fig:sgdstep_landscape}
\end{figure}

\section{Certifying distributional robustness}
We now detail the theoretical guarantees for our \sdro algorithm.
Our analysis is based on an insight on the connection between our robustification scheme using the \ktran~\eqref{eq:ktrans} and OT.
We first show that \sdro can be viewed as a robustification and convexification scheme in the log-transformed space.
All proofs are given in the appendix.
\subsection{Robustification in log scale}
\newcommand{\lnlsighat}{\ensuremath{\widehat{\ln l_{1/\sigma}}}}
We rewrite the surrogate loss $l^k_\var(\xi_i)$ of the inner optimization problem in \eqref{eq:ksmooth} by simplying taking the log transform, obtaining
\begin{equation}
\label{eq:intro_ctran_lnl}
    l^k_\var(\xi_i) 
    =
    \exp \sup_u
    \Bigg\{
     \ln l(u)
     -\frac{1}{\sigma}c(u,\xi_i)
    \Bigg\}.
\end{equation}
We exchanged the $\sup$ and $\exp$ above due to the monotonicity and continuity of the exponential function.
Using the above relationship, 
we rewrite \sdro as the equivalent optimization problem
\begin{multline}
    \label{eq:log}
    \min_\var
    \frac{1}N\sum_{i=1}^N 
    \exp
    \Bigg\{
    \lnlsighat (\var, \xi_i):=
    \\
    \sup_u
    \big\{
    \ln l(u)
    -\frac{1}{\sigma}c(u,\xi_i)
    \big\}
    \Bigg\}
    ,
\end{multline}
where
$\lnlsighat (\var, \cdot )$ denotes the (negative) \ctran of the log-loss $\ln l(\var, \cdot)$.
The following lemma states that the inner maximization objective is concave for certain choices of the kernel bandwidth.
Hence, standard analysis such as \citep{linGradientDescentAscent2021} applies to our setting.

A function $f$ on $\mathbb R^d$ is said to be L-smooth if
    $
	\|\nabla f(x)-\nabla f(y)\|_{2} \leq L\|x-y\|_2, \forall x,y\in\domain,
    $
provided that all quantities exist.
\begin{proposition}
    [Convexification]
    Suppose the function $\ln l(\var, \cdot)$ is L-smooth, transport cost $c$ is 1-strongly convex, and $\sigma < \frac{1}{L}$.
    Then, the inner maximization objective of \eqref{eq:log} is strictly concave.
    \label{thm:convexify}
\end{proposition}
In particular, Gaussian kernels with appropriately chosen bandwidths satisfy the assumptions in Proposition~\ref{thm:convexify}.
The intuition is that, by moving the exponentiation outside, we see the convexification mechanism of \sdro more clearly:
compared with \wdro, \sdro can be viewed as applying the \ctran as robustification in log scale.
The log-transform in \eqref{eq:log} also gives us an intuition for choosing the kernel bandwidth $\sigma$ using the practice in proximal algorithms (Moreau-Yosida regularization).
Alternatively, \eqref{eq:log} can also be seen as optimizing a (differentiable) \emph{softmax} version of the worst-case loss $\sup_{\xi\in\domain}\lnlsighat (\var, \xi)$.

\subsection{Certifying robustness under distribution shift using OT}
\label{sec:theory}
\newcommand{\raten}{\ensuremath{r_{N,\delta}}} %
So far, it is not immediately clear how we can produce a certificate for the amount of robustness against distribution shift, or how to measure the distribution shift, e.g., in what metric?
This section exploits the connection of our proposed \ktran and the \ctran to provide answers to those questions.

Specifically, we bound the quantity when considering an arbitrary distribution $P$ that differs from the true data-generating distribution \ptrue.
We then show that the \sdro procedure \eqref{eq:ksmooth} produces a robustness certificate in the setting of distribution shift.

In the following, we assume that the \ctran $\lnlsighat (\var, \xi)$ (see \eqref{eq:log}) is bounded, i.e., $\exists M>0$ such that
    $
    | \lnlsighat (\var, \xi) |\leq M, \forall \xi\in\domain.
    $
As we have already shown that \lnlsighat{} is a \emph{majorant} of $\ln l$, the assumption also implies
	$
    l (\var, \xi) \leq e^{M}.
    $
Intuitively, our analysis of \sdro certifies the robustness in log scale.
Nonetheless, since the log function is monotone, we can still use OT distances to control the generalization under distribution shift.
To ease the notation, put $\raten:=\sqrt{\frac{\ln(1/\delta)}{n}}$.
\newcommand{\radm}{\ensuremath{\operatorname{\mathcal{R}_N}}}
\newcommand{\radmacherLog}{\ensuremath{\radm(\{\lnlsighat(\var, \cdot) | \var \in \Theta\})}}
\begin{proposition}
    [Certifying robustness against distribution shift]
    For any \var pointwise, 
    there exists a constant $C>0$ such that,
    $\forall \rho > 0$, any probability measure $P$, and kernel bandwidth $\sigma>0$, the following holds except probability $\delta$:
    \begin{multline}
    	    	\label{eq:certificate}
        \sup_{\gamma(P,\ptrue)\leq\rho}
        \expected{P}{\ln l (\var, \xi)}
        \leq\\
        \ln
        \Bigg\{
        \underbrace{
        \frac1N\sum_{i=1}^N l^k_\var(\xi_i)
        }_{\textrm{\sdro{} objective}}
        \Bigg\}
        + \frac{\rho}{\sigma}
        +C\cdot \raten,
    \end{multline}
    where $\gamma$ is the \wsd associated with transport cost $c$.

    Furthermore,
    there exits a constant $C'$ that does not depend on \var
    such that
    the following holds except probability $\delta$:
    \begin{multline}
        \sup_{\gamma(P,\ptrue)\leq\rho}
        \expected{P}{\ln l (\var, \xi)}
        \leq
        \\
        \ln
        \Bigg\{
        \frac1N\sum_{i=1}^N l^k_\var(\xi_i)
        \Bigg\}
        + \frac{\rho}{\sigma}+C'\cdot \raten
        \\
        +2\cdot \radmacherLog
        ,
    \end{multline}
    where \radm denotes the Rademacher complexity.
    \label{thm:generalization}
\end{proposition}
The first two (non-diminishing) terms of the bound in \eqref{eq:certificate} RHS,
$\ln
\Bigg\{
\frac1N\sum_{i=1}^N l^k_\var(\xi_i)
\Bigg\}
+ \frac{\rho}{\sigma}$, give us a \emph{computable robustness certificate} under the distribution shift from $\ptrue$ to arbitrary $P$.
This is in line with the robustness certificate of \citep{sinhaCertifyingDistributionalRobustness2017,leeMinimaxStatisticalLearning2018a}, and different from typical statistical learning theory bounds.
Furthermore, the robustness certificate is simply the log-transform of the \sdro objective plus the \emph{regularization} term $ \frac{\rho}{\sigma}$.

\begin{remark}
    The bound for Rademacher complexity of common function classes is a well-studied topic in statistical learning theory.
    It also follows the Lipschitz composition rule.
    For example, for model classes such as RKHS functions with bounded norms $\{f\in\rkhs|\hnorm{f}\leq R\}$ for $R>0$ and bounded kernels, the Rademacher complexity decays at the rate of $\frac{1}{\sqrt{N}}$.
    However, as mentioned earlier, we are interested in the function spaces that are more general and hence do not further expand on bounding the Rademacher term using specific spaces.
    We refer to more specialized texts such as \citep{vaartWeakConvergenceEmpirical2013} for more details and \citep{sinhaCertifyingDistributionalRobustness2017} for a recent application to robustness certificate.
\end{remark}

\paragraph{Unifying DRO using OT and kernel methods}
    Our analysis above establishes a non-trivial connection between two branches of DRO research using OT (e.g., \citep{mohajerinesfahaniDatadrivenDistributionallyRobust2018,sinhaCertifyingDistributionalRobustness2017}) and kernel methods \citep{zhuKernelDistributionallyRobust2020d,staibDistributionallyRobustOptimization2019}. In the center stage is the kernel bandwidth parameter $\sigma$.
    
    From the OT perspective,
    we see that larger $\sigma$ in \sdro corresponds to smaller scaling parameters in the c-transform for OT, which is known to lead to more conservatism in \wdro (since we down-weigh the transportation cost, resulting in larger ambiguity region). 
    
    On the other hand, and from the kernel perspective, the authors of \citep{zhuKernelDistributionallyRobust2020d} use functional analysis arguments to characterize that conservative \kdro can be a consequence of using small RKHSs as the dual spaces for DRO, which are associated with kernels with larger bandwidth $\sigma$.
	
	In summary, through the bandwidth parameter $\sigma$, this paper unifies the DRO performance-robustness trade-off for both OT and kernel methods.

\section{Numerical Experiments}
\label{sec:exp}
In this section, we empirically demonstrate that \sdro can easily work with DNN models, which is a limitation of typical existing DRO reformulation techniques.
Our selection of neural architectures is meant to demonstrate the algorithmic robustification effect instead of achieving state-of-the-art benchmarks.
To that end, we also ablate factors known to influence robustness, such as dropout if not specified.
More details on our experimental setup, hyper-parameter selection, and additional experimental results can be found in the appendix.
\begin{figure}[b!]
	\centering
    \vspace{-0.5cm}
	\begin{subfigure}{0.49\textwidth}
        \centering
        \includegraphics[width=2.7in,valign=c]{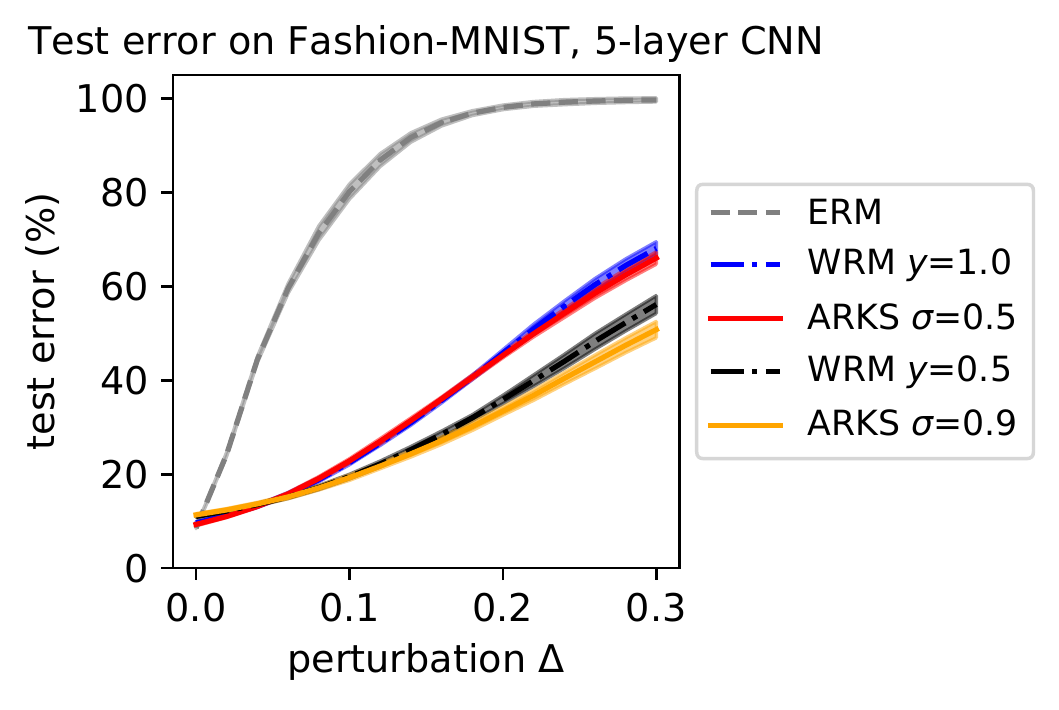}
    \end{subfigure}
    	\begin{subfigure}{0.49\textwidth}
        \centering
        \includegraphics[width=2.7in,valign=c]{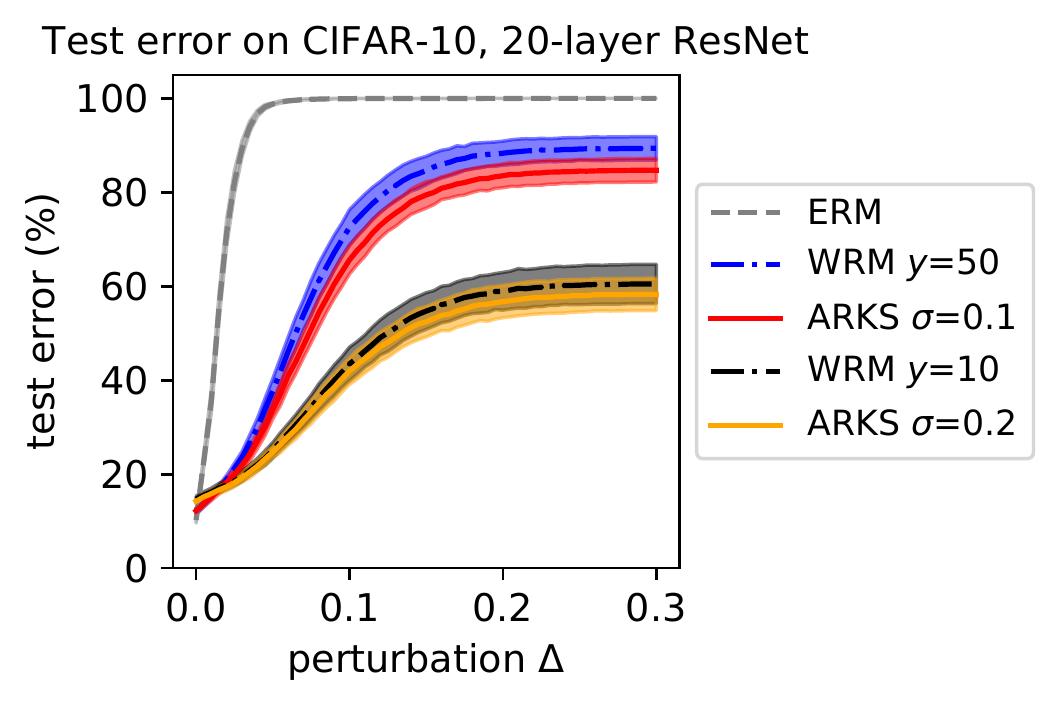}
    \end{subfigure}
    \caption{Black-box PGD attack with respect to $\| . \|_\infty$ on the Fashion-MNIST (\textbf{top}) and CIFAR-10 (\textbf{bottom}) datasets. We show the classification error on perturbed test images versus the allowed magnitude of the adversarial perturbation  $\Delta$. \sdro and WRM exhibit similar adversarial performance profiles; \sdro becomes more robust as the kernel width $\sigma$ increases, while WRM improves with a lower Lagrangian penalty $y$. For all algorithms, we report the mean and standard deviation across 10 random seeds.}
    \label{fig:pgd_attacks}
\end{figure}
\subsection{Robust Learning with DNNs}
\label{subsec:exp_adv}
We compare the following algorithms when applicable: \emph{(A)} \sdro~Algorithm~\ref{alg:smooth}, \emph{(B)} empirical risk minimization (ERM) and \emph{(C)} WRM~\citep{duchiStatisticsRobustOptimization2018}, as well as \emph{(D)} projected gradient descent (PGD) for training \citep{madryDeepLearningModels2019} (reported in the appendix since it is based on RO instead of DRO; We also refer to \citep{sinhaCertifyingDistributionalRobustness2017} for extensive comparisons of PGD against WRM) and \emph{(E)} (worst-case) robust optimization~\citep{ben-talRobustOptimization2009a,soysterTechnicalNoteConvex1973} (reported in the appendix since it is not applicable for deep learning tasks).
We do not test classical \wdro algorithms, e.g. \citep{mohajerinesfahaniDatadrivenDistributionallyRobust2018,zhaoDatadrivenRiskaverseStochastic2018,shafieezadeh-abadehRegularizationMassTransportation2019}, since they cannot be applied to our test settings with general losses and DNN models.
We further note that the classical type-2 \wdro reformulation is equivalent to WRM with the optimal dual variable.

In our evaluation, the test data is perturbed with worst-case disturbances $\delta$ within a box $\{\delta: \| \delta\|_{\infty} \leq \Delta  \}$. $\delta$ is generated by attacking the model trained with ERM (for each random seed) using the PGD algorithm. This type of attack is referred to as \textit{black-box}. The experiment is also performed with black-box fast-gradient sign method (FGSM) \citep{goodfellowExplainingHarnessingAdversarial2015} attacks with respect to $\| .\|_{\infty}$, as well as white-box PGD attacks; the results are included in the appendix. We note that our focus is to demonstrate the robustification effect of the algorithm in a known environment instead of benchmarking various attacks exhaustively.

\textbf{Fashion-MNIST \citep{xiao2017fashionmnist} with CNN.}
The top panel of Figure~\ref{fig:pgd_attacks} shows the classification error for increasing perturbation magnitude $\Delta$. We observe that ERM attains good performance when there is no perturbation but quickly underperforms as $\Delta$ increases. \sdro and WRM yield improved robustness while also achieving low test error under no perturbation. 
\sdro and WRM exhibit similar performance profiles; see the caption of Figure~\ref{fig:pgd_attacks}.
To conclude, \sdro performs at least competitively with WRM.

\textbf{CIFAR-10 \citep{cifar} with ResNet-20.}
The above experiment is repeated for the CIFAR-10 dataset. We choose a deeper architecture, the ResNet-20 \citep{he2015deep} with batch normalization \citep{ioffe2015batch} and ReLU activations. During training, we noticed that WRM might require tuning $y$ to be arbitrarily large for stable performance under highly non-smooth losses, while $\sigma$ can be easily tuned within a small range. The results are shown on the bottom panel of Figure~\ref{fig:pgd_attacks}: \sdro exhibits improved robustness under adversarial perturbations with little to no training performance sacrifice and is at least competitive with WRM.

\textbf{CelebA \citep{liu2015faceattributes} with CNN.}
Experimental results -- similar to other datasets -- can be found in the appendix, while Figure~\ref{fig:celeba_imgs} illustrates examples of perturbed images generated during training. 
\begin{figure}[h!]
	\centering
	\begin{subfigure}{0.49\textwidth}
        \centering
        \includegraphics[width=2.7in,valign=c]{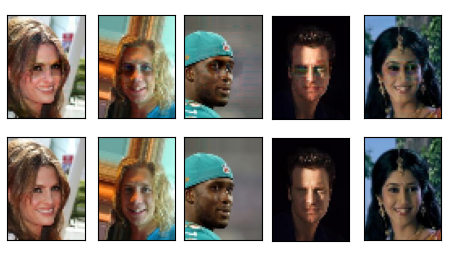}
    \end{subfigure}
    \caption{\textbf{(top)} Perturbed images ($u^{*}$) maximizing the inner optimization of \sdro in CelebA binary classification. \textbf{(bottom)} Unperturbed counterpart. We observe that \sdro generates worst-case perturbations by creating interference around the eyes, reducing apparent separation between the two classes (with or without eye-wear).}
    \label{fig:celeba_imgs}
\end{figure}

\section{Discussion}
\label{sec:discus}
In this paper, we propose the \sdro algorithm using tools from convex analysis, kernel smoothing, and robust optimization.
We have demonstrated state-of-the-art performance in benchmarks of learning under distribution shifts, especially with DNN models that can not be treated with typical \wdro convex reformulation techniques.
Furthermore, we have provided guarantees that certify robustness under distribution shift.

A future direction is to design specific kernels for robust learning, especially for incorporating rich descriptions of interventions in the real world beyond the typical norm-ball perturbation.
For example, we can similarly design the transport cost in our \ckrnl (Definition~\ref{def:ckernel}) to be the data-dependent Mahalanobis distance to protect against distribution shift for causal inference as in, e.g., \citep{heinze-demlConditionalVariancePenalties2021}.

\section*{Acknowledgements}
We thank the anonymous reviewers for their constructive comments during the review process.
We also thank 	
Simon Buchholz
for sending us helpful feedback on the initial manuscript.
This project received support from the German Federal Ministry of Education and Research (BMBF): Tübingen AI Center, FKZ: 01IS18039B.

\bibliography{bib_idro, bib_idro2, bib_neurips}
\newpage
\onecolumn
\appendix
\begin{center}
  {\LARGE{}Appendix: Kernel Robust Smoothing}{\LARGE\par}
\par\end{center}
\newcommand{\lu}{l(u)}
\newcommand{\ku}{k(u,x)}
\newcommand{\df}{\frac{d}{du}f(u)}
\newcommand{\dl}{\frac{d}{du}l(u)}
\newcommand{\dk}{\frac{d}{du}k(u,x)}
\newcommand{\ddf}{\frac{d^2}{du^2}f(u)}
\newcommand{\ddl}{\frac{d^2}{du^2}l(u)}
\newcommand{\ddk}{\frac{d^2}{du^2}k(u,x)}
\newcommand{\rbf}{e^{-{(u-x)^2}/2\sigma }}

\paragraph*{Notation and background.}
Throughout the appendix, we will consider the cases where the inner suprema of minimax problems are attained.
Without further specifications, we consider the default kernel choice to be the \gkrnl $k_\sigma(u,x)=e^{-{{\|u-x\|}^2_2}/2\sigma}$ (or the \lapkrnl) in the rest of the appendix. We suppress the kernel bandwidth $\sigma$ when there is no ambiguity in the context.

\section{Proofs of theoretical guarantees}
\subsection{Proof of Proposition~\ref{thm:convexify}}
\begin{proof}
	The proof is an exercise of calculus, e.g., by using the Taylor expansion of 
        $
        \ln l(u)
        -\frac{1}{\sigma}c(u,\xi_i)
        $ w.r.t. the variable $u$.
\end{proof}

\subsection{Proof of Proposition~\ref{thm:generalization}}
As preparation, we first establish a standard concentration result in this paper's context.
\begin{lemma}
        [Concentration]
        For any \var pointwise, 
        there exists constant $C>0$ such that,
        $\forall \rho > 0$, probability measure $P$, and kernel bandwidth $\sigma>0$, the following holds except probability $\delta$:
        \begin{equation*}
            \expected{\ptrue}{\lnlsighat (\var, \xi)}
            \leq
            \frac1N\sum_{i=1}^N \lnlsighat (\var, \xi_i)
            +C\cdot\raten.
        \end{equation*}
        \label{thm:moreau_concentrate}
\end{lemma}
\begin{proof}
Since the function of interest $\lnlsighat(\theta, \cdot)$ satisfies the bounded difference condition by assumption, 
the lemma statement follows directly from the McDiarmid's inequality.
\end{proof}

We now prove Proposition~\ref{thm:generalization}.
\begin{proof}
        By the strong duality of DRO using \wsd (see, e.g., \cite{mohajerinesfahaniDatadrivenDistributionallyRobust2018,gaoDistributionallyRobustStochastic2016,zhaoDatadrivenRiskaverseStochastic2018}) and the \ctran notation,
        we have, $\forall \sigma > 0, \rho >0, \var\in\Theta$,
        \begin{equation}
            \label{eq:duality_concentration}
            \sup_{\gamma(P,\ptrue)\leq\rho}
            \expected{P}{\ln l (\var, \xi)}
            =
            \inf_{\sigma > 0}
            \Bigg\{
                \expected{\ptrue}\lnlsighat(\var, \xi)
                + \frac{\rho}{\sigma}      
            \Bigg\}
            \leq
            \expected{\ptrue}\lnlsighat(\var, \xi)
            + \frac{\rho}{\sigma}.
        \end{equation}
        By Lemma~\ref{thm:moreau_concentrate}, for any fixed \var, the RHS above is bounded by
        \begin{equation*}
            \frac1N\sum_{i=1}^N \lnlsighat (\var, \xi_i)
            + \frac{\rho}{\sigma}
            +C\cdot \raten
            \leq
            \ln\frac1N\sum_{i=1}^N \exp\Bigg\{\lnlsighat (\var, \xi_i)\Bigg\}
            + \frac{\rho}{\sigma}
            +C\cdot \raten
            ,
        \end{equation*}
        where the inequality is due to the concavity of the log function.
        Noting the quantity inside the logarithm is simply the \sdro objective (see, e.g., \eqref{eq:log}), we obtain the first half of the proposition statement.

        We now show the uniform convergence.
        Put
        \begin{equation*}
            F_N : = \sup_{\var\in\Theta} 
            \expected{\ptrue}\lnlsighat(\var, \xi)
            -\frac1N\sum_{i=1}^N \lnlsighat (\var, \xi_i)
            .
        \end{equation*}
        Using symmetrization (see, e.g., \cite{vaartWeakConvergenceEmpirical2013}), we have $$\expected{\ptrue}{F_N}\leq 2 \radmacherLog.$$
        By McDiarmid, with $1-\delta$ probability,
        \begin{equation*}
            F_N \leq \expected{\ptrue}{F_N} + C'\cdot \raten,
        \end{equation*}
        where $C'$ does not depend on \var due to the $\sup$ operation.
        Combining the above relationships with \eqref{eq:duality_concentration} yields the uniform bound.
    \end{proof}

\section{Additional technical details}
\subsection{Equivalence of type-1 \wdro to IPM-DRO}
\begin{lemma}(Motivating example using type-1 \wdro)
	\label{thm:w1envelope}
	Suppose the loss function $l(\var,\cdot)$ is $y-$Lipschitz continuous.
	Let variable $f$ in dual IPM-DRO~\eqref{eq:kdro_dual} be set to the $y$-\phenv $f(\cdot):=\sup_u\{
		l(\var, u) - y\cdot \|u - \cdot\|
		\}$.
	Then, \eqref{eq:kdro_dual} is equivalent to the dual formulation of type-$1$ \wdro; cf. \citep{mohajerinesfahaniDatadrivenDistributionallyRobust2018,zhaoDatadrivenRiskaverseStochastic2018,kuhnWassersteinDistributionallyRobust2019}.
\end{lemma}
Unfortunately, estimating Lipschitz constants for general model classes is known to be difficult \citep{virmauxLipschitzRegularityDeep2018,biettiKernelPerspectiveRegularizing2019},
resulting in the intractability of \wdro when used with common machine learning models, e.g., neural networks, which our method in this paper can handle.

\subsection{Proof of Lemma~\autoref{thm:w1envelope}}
We now prove Lemma~\autoref{thm:w1envelope}. First, it is an exercise to show the following technical lemmas.
\subsubsection{Technical lemma and proofs using convex analysis}
\begin{lemma}
    \label{thm:sup_conv}
    A function's $y$-\phenv dominates itself, i.e.,
    $$
    l_{y,1} (x) \geq l(x), \forall x\in\domain.
    $$
    Furthermore,
    $l_{y,1}$ is the smallest majorant of $l$ with Lipschitz constant $y$.
\end{lemma}
\begin{lemma}
\label{thm:sup_conv_coincide}
    If $l$ is Lipschitz-continuous with constant $y$, then $l_{y,1}$ coincides with $l$.
\end{lemma}

\newcommand{\ly}[1]{\ensuremath{l_{y, {#1}}}\xspace} %
\newcommand{\ymajor}{$y$-Lipschitz majorant\xspace}
Similar results concerning the \emph{infimal convolution} (instead of supremal) are well-known \cite[Chapter 12]{bauschkeConvexAnalysisMonotone2011}.
For completeness, we give self-contained proofs below.
We assume the regularity condition that
$f(x):=\ly1(x)=\sup_u\{
        l(\var, u) - y\cdot \|u - x\|
\}
< \infty
$; we refer to \cite[Proposition~12.14]{bauschkeConvexAnalysisMonotone2011} for the degenerative case when $\ly1=\infty$ and $l$ has no \ymajor.

We now prove Lemma~\ref{thm:sup_conv}.
\begin{proof}
        By noting the special choice of $u=x$, the relationship $f(x)\geq l(x)$ is obvious. We now prove the Lipschitz continuity.

        For any $x,z$ in the domain,
\begin{multline}
        f(x) =\sup_u\{
                l(u) - y\cdot \|u - x\|\} 
                \geq                \sup_u\{
                        l(u) - y\cdot \|u - z\| - y\cdot\|z - x\|\}
                        =\sup_u\{
                                l(u) - y\cdot \|u - z\|\} - y\cdot\|z - x\|\\
                                = f(z) -y\cdot\|z - x\|.
\end{multline}
Therefore,
        $
        f(z) - f(x)\leq y\cdot\|z - x\|,
        $
$f$ is $y$-Lipschitz.

To show that $f$ is the smallest \ymajor, we let $g$ be any \ymajor of $l$.
Then,
$$
g(x)\geq g(z) - y\cdot\|z - x\| \geq l(z) - y\cdot\|z - x\| .
$$

Take supremum on both sides,
$$
g(x) \geq\sup_z\{ l(z) - y\cdot\|z - x\|\}= f(x) .
$$
Hence, $f$ is the smallest \ymajor.
\end{proof}
Lemma~\ref{thm:sup_conv_coincide} follows directly from Lemma~\ref{thm:sup_conv}.
See, e.g., \cite{bauschkeConvexAnalysisMonotone2011} Chapter 12 for more technical details on the convolution operator.

By plugging in the expression for $l_{y,1}$, we have
\begin{equation}
        \begin{aligned}
                \frac{1}N\sum_{i=1}^N \sup_u\{l(u) - y\cdot \|u - \xi_i\|\}
                + \epsilon y
        \end{aligned}
\end{equation}
We have thus recovered the type-1 Wasserstein DRO dual
as a special case of our analysis.

\subsection{Proof of Proposition~\ref{thm:majorant}}
\newcommand{\ustar}{u^*}
We now verify the relationship
$
        \fsig(x)\geq l(x), \forall x\in\domain,\text{ and } \fsig \to l \textrm{ as } \sigma \to 0.
$
The dominance relationship $\fsig(x)\geq l(x)$ can be seen by taking the special case $u=x$ in the supremum.
Finally, the convergence of $\fsig \to l \textrm{ as } \sigma \to 0$ is obvious by examining the expression of the \gkrnl and \lapkrnl.
\subsection{Proof of Proposition~\ref{thm:recover_ro} (Robustness-performance trade-off using kernel width $\sigma$)}
First, we note the continuity of the \gkrnl and the loss function $l$; hence all limits are attained.
If we let
the kernel width be large $\sigma\to\infty$, 
then
$
 \lim_{\sigma\to\infty} \ku = 1.
$
Hence, the robust learning algorithm recovers the worst-case robust optimization (RO)
$$
\min_\var \sup_\xi l(\var, \xi).
$$

Similarly, if kernel width is small $\sigma\to 0$, then we recover the trivial 
Dirac function at limit
$
 \lim_{\sigma\to0} \ku = \delta_x(u).
$
Hence \sdro becomes the empirical risk minimization (ERM),
$$
\min_\var \frac{1}N\sum_{i=1}^N  l(\var, \xi_i).
$$

\subsection{Alternative analysis on convexity properties of inner optimization problem}
We now provide an alternative view (to Proposition~\ref{thm:convexify}) of
the convexity properties of the objective function of the inner objective of \sdro,
which we denote as $f(u):= l(u)k(u,x)$.
For \sdro, our intuition is that, by multiplying the loss $l(u)$ by a function $k(u,x)$ which is strongly concave near its peak, the resulting function is consequently locally concave too.
This idea is illustrated in Figure~\ref{fig:sgdstep_landscape}.
For conciseness, we assume that the loss function $l$ is positive twice-differentiable (cf. \cite{sinhaCertifyingDistributionalRobustness2017} for why this is not restrictive), and $x, u$ are scalars.
We first show that the inner objective $f(u)=l(u)k(u,x)$ is locally concave in a neighborhood of $x$.
\begin{proof}        
We compute the curvature $\ddf$.
\begin{multline}
    \label{eq:hessian}
    \ddf =\frac{d}{du} \left(\dl \ku + \lu \dk\right),\\
    = \ddl \ku + 2\dl\dk + \lu\ddk\\
    = \rbf \left[ \ddl + 2\dl \left(-{(u-x)}/\sigma\right)
    + \lu\left( -1/\sigma +  (u-x)^2/\sigma^2  \right)
    \right].
\end{multline}

Let us choose $\sigma>0$ small enough such that the following holds.
\begin{equation}
        \frac{d^2}{du^2}l(u) - \lu/\sigma < 0.\label{eq:curvature1}
\end{equation}
This can be done trivially if the curvature of the loss $l$ is bounded (similar to the assumptions in \cite{sinhaCertifyingDistributionalRobustness2017,houskaNonlinearRobustOptimization2013}) and $l(u)>0$.
Then, there exists $\Delta>0$ such that, for $|u-x|\leq\Delta$, the curvature value~\eqref{eq:hessian} is negative. Therefore, the objective $f(u)=l(u)k(u,x)$ is concave in the $\Delta-$neighborhood of $x$.
\end{proof}

We now show that, for a suitable choice of $\sigma$, every stationary point of $f$ is a local maximum, hence explaining the good empirical performance in our experiments.
A full convergence analysis is out of the scope of our current paper.

Let
$$
\sigma^* = \frac{2{(\ustar-x)}^2}{\sqrt{1+4{(\ustar-x)}^2\cdot \frac{d^2}{du^2}l(\ustar)/l(\ustar)}-1},
$$
which is a non-negative quantity if $\ustar\neq x$ and $\frac{d^2}{du^2}l(\ustar) > 0$ by straightforward verification.
\begin{lemma}
        Suppose either the loss $l$ is concave or the bandwidth satisfies $\sigma<\sigma^*$.
        Then, every stationary point of $f$ is a maximum.
\end{lemma}
\begin{proof}
        Suppose $\ustar$ is a stationary point of $f$, which implies
        \begin{equation*}
                \df\mid_{u=\ustar}=\dl \ku + \lu \dk\mid_{u=\ustar}=0.
        \end{equation*}
        Since $\ku\neq 0$, that further implies
        \begin{equation*}
                \dl  + \lu \left(-{(u-x)}/\sigma\right)\mid_{u=\ustar}=0.
        \end{equation*}
        Plugging the above equality into the last line of \eqref{eq:hessian},
        \begin{equation*}
                \ddf\mid_{u=\ustar}
                = \rbf \left[ \ddl 
    - \lu\left( 1/\sigma +  (u-x)^2/\sigma^2  \right)
    \right]\mid_{u=\ustar}.
        \end{equation*}
        Since either $\frac{d^2}{du^2}l(\ustar)\leq 0$ or $\sigma<\sigma^*$, we have
        \begin{equation*}
                \frac{d^2}{du^2}l(\ustar)
                - l(\ustar)\left( 1/\sigma +  (\ustar-x)^2/\sigma^2  \right) < 0.
        \end{equation*}
        Then,
        \begin{equation*}
                \ddf\mid_{u=\ustar} < 0.
        \end{equation*}
        Therefore, $\ustar$ is a local maximum by the second derivative test of calculus.
\end{proof}
Note that the condition $\sigma<\sigma^*$ can be easily satisfied when $l$ has bounded curvature and is a weaker condition than \eqref{eq:curvature1}.
Therefore, the above lemma implies that a gradient-based algorithm converges to a maximum.
Note that the analysis presented in Proposition~\ref{thm:convexify} is stronger than the lemma above.

\subsection{Additional function approximator and majorant constructions}
We now use smooth majorants, as well as interpolants, to construct robustification methods in addition to the \sdro, e.g.,
\begin{enumerate}[noitemsep,topsep=0pt]
        \item Kernel distance envelope (\cdro)
                \begin{equation}
                        \label{eq:kernelconv}
                        f_{\sigma,y}(x):=\kconv{x}.
                \end{equation}
        \item Kernel interpolant (\idro; $l$ denotes the vector of loss values at some interpolation points
                                $[l(\theta,\xi_1), \dots, l(\theta,\xi_M)]^\top$
                                )
                \begin{equation}
                        \label{eq:interpolant}
                        \hat f(x) = l^\top k(X, X)^{-1} k (X, x).
                \end{equation}
                While \idro is not a strict majorant (since it only interpolates at data sites), we nonetheless show below it can enforce robustness.
\end{enumerate}
Once we take the function approximation perspective, the possibility is by no means limited to those choices. For example, for the inverse multi-quadratic kernels, the approximation
$        
\sup_u \{
                l(\var,u) - y[ 1/ k^2(u,{x}) - C^2])
                \}
$
is equivalent to using the \mreg in type-2 \wdro.
The authors of \citep{zhuKernelDistributionallyRobust2020d} used the RKHS basis expansion
$
\hat f(x) =  \sum_{j=1}^M \alpha_j k (\zeta_j, x)
$,
for some discretization points $\zeta_j$.
Compared with their approach, our choices in \eqref{eq:kernelconv} and \sdro are certified majorants of loss function $l$.
We now examine the specific approximation schemes.

\paragraph*{Kernel distance envelope (\cdro).}
We adopt a similar insight as \wdro in \citep{sinhaCertifyingDistributionalRobustness2017} and propose the function approximator, which is a variant of the \mreg with the kernel distance (using the \gkrnl)
\begin{equation}        
        f_{y,\sigma}(x) = \sup_u \{
                l(\var,u) - y/2\cdot \hnorm{\phi(u) - \phi (x)}^2
                \}
                = \kconv{x}.
\end{equation}
Similar to \sdro, one can verify
\begin{equation*}
        f_{y,\sigma}(x)\geq l(x), \forall x\in\domain,\text{ and } f_{y,\sigma} \to l \textrm{ as } y \to 0,
\end{equation*}
i.e., it is a function majorant of $l$.
Furthermore, $f_{y,\sigma}$ can be viewed as a $y-$Lipschitz continuous mapping from the feature space (i.e., the RKHS \rkhs; see the appendix),
$$
\begin{aligned}
        f_{y,\sigma}(x)-f_{y,\sigma}(z)\leq y \hnorm{\phi{(x)}-\phi{(z)}},\ \forall x,z\in\domain.
\end{aligned}
$$
Analogous to \sdro, we can simply solve the optimization problem with a fixed $y$,
\begin{equation}
        \label{eq:smooth_dro}
        \begin{aligned}
                \min_\var
                \frac{1}N\sum_{i=1}^N 
                \kconv{\xi_i}.
        \end{aligned}
\end{equation}

Similar to \eqref{eq:ksmooth},
for suitable choices of $y,\sigma$,
the inner function of \eqref{eq:smooth_dro}
is locally concave in a neighborhood of $\xi_i$, facilitating gradient based optimization.
Note that \cdro can again be interpreted as adversarial training like \sdro.
The empirical performance of \cdro is similar to \sdro in our experiments.
We thus left it as future work to examine the properties of \cdro in detail.

\paragraph*{Kernel interpolant (\idro).}
We now turn to the \emph{kernel interpolant}, an entirely different scheme from \sdro.
Our main idea in this section is to find a map $L^2\to \rkhs$ such that we can perform robust learning in $\rkhs$. 
For any given \var, we choose the approximation function $f$ to be
the well-known kernel interpolant \citep{wahbaSplineModelsObservational1990} of the loss function
$$
\hat f = l^\top k(X, X)^{-1} k (X, \cdot),
$$
where $l$ is defined in \eqref{eq:interpolant}.
This is also referred to as the kernel ``ridge-less'' regression estimator.
Plugging this interpolant back into the 
IPM-DRO formulation,
we arrive at the regularized risk minimization
\begin{equation}
        \label{eq:kidro}
        \begin{aligned}
                \min_{ \var}
                \frac{1}N\sum_{i=1}^N l(\var, \xi_i)
                + \epsilon\sqrt{l^\top k(X, X)^{-1} l  }.
        \end{aligned}
\end{equation}
Intuitively, this can be seen as performing the following two steps simultaneously: \emph{1)} interpolating the optimization loss $l$ using kernel regression; and \emph{2)} performing regularized risk minimization w.r.t. \var using the interpolant function's RKHS norm.

Alternatively, using the least-squares loss as an example,
$
l(\var,[X,Y]):= (g_\var(X)-Y)^2,
$
we may use $f$ to interpolate the model $g_\var$ only, resulting in
\begin{equation}
\min_{ \var}
\frac{1}N\sum_{i=1}^N l(\var,[x_i,y_i]) + \epsilon \sqrt{[(f(X)-Y)^2]^\top k(X, X)^{-1}[(f(X)-Y)^2]  },
\label{eq:krr}
\end{equation}
where $ f = g_\var(X)^\top k(X, X)^{-1} k (X, \cdot)$. %
In practice, we may also choose to use a regularizer motivated by kernel ridge regression,
$$
\begin{aligned}
\min_{ \var}
\frac{1}N\sum_{i=1}^N (g_\var(X_i)-Y_i)^2 + \lambda{g_\var(X)^\top k(X, X)^{-1} g_\var(X)  }.
\end{aligned}
$$
We refer to \citep{biettiKernelPerspectiveRegularizing2019,xuRobustnessRegularizationSupport2009,staibDistributionallyRobustOptimization2019} for more interpretations of RKHS norm regularization.
\begin{remark}
        While the above formulations, such as \eqref{eq:krr}, resemble the kernel ridge regression (KRR) estimator, they are \textit{not} not the same. Our method can learn with either parametric or non-parametric models with loss $l(\theta, \cdot)$, while KRR only works with kernelized models. For example, we have report experiments with DNNs, which cannot be handled by KRR.
\end{remark}

\section{Experimental set-up and additional results}
\label{sec:appendixc}
In this Section, we provide additional information on the numerical experiments presented in Section~\ref{sec:exp}, and report supplementary material. This includes benchmark results with the PGD adversarial learning algorithm \citep{madryDeepLearningModels2019}, and a brief discussion on its differences with ARKS. All of our experiments are conducted using the PyTorch \citep{NEURIPS2019_9015} and the CVXPY \citep{diamond2016cvxpy} libraries.

\subsection{Robust learning under Adversarial Perturbations}

\paragraph{Datasets.} 
The numerical experiments in Section~\ref{subsec:exp_adv} make use of the following publicly available datasets: Fashion-MNIST\footnote{available at \url{https://pytorch.org/vision/stable/datasets.html\#fashion-mnist}} \citep{xiao2017fashionmnist}, CIFAR-10\footnote{available at \url{https://pytorch.org/vision/stable/datasets.html\#cifar}} \citep{cifar}, and CelebA\footnote{available at \url{https://www.kaggle.com/jessicali9530/celeba-dataset}} \citep{liu2015faceattributes}.

The Fashion-MNIST dataset contains greyscale images of garments from 10 categories. Each image $x$ is represented by $x \in [0, 1]^{28 \times 28}$.

The CIFAR-10 dataset contains colored images of different objects from 10 categories. Each image x is represented by $x \in \left[0, 1 \right]^{32 \times 32 \times 3}$.
As is customary in such settings, we augment the training set with additional samples by randomly cropping and flipping images. 

 Using the provided attributes, the CelebA dataset is reduced to only contain a balanced number of colored images of celebrities with eye-wear (class 1) or without (class 0). Each image $x$ is represented by $x \in [0, 1]^{64 \times 48 \times 3}$. Our codebase includes a script to modify this dataset.

\paragraph{Model architectures.} For Fashion-MNIST, the model architecture consists of two $3 \times 3$ convolutional layers with ELU activations and max pooling, followed by two fully connected layers and a softmax layer. For CIFAR-10 we use the ResNet-20 model architecture \citep{he2015deep}, consisting of 18 convolutional layers with batch normalization \citep{ioffe2015batch} and ReLU activations, followed by a fully connected and a softmax layer.  For CelebA, the model architecture is borrowed from \cite{heinze-demlConditionalVariancePenalties2021}, comprised of four $5 \times 5$ convolutional layers with Leaky ReLU activations, followed by a fully connected and a softmax layer. The convolutions produce 16, 32, 64 and 128 channels respectively, using a stride of 2.

\paragraph{Hyper-parameter search.}
To optimize our algorithms, we performed grid search over the hyper-parameters outlined in Table \ref{tab:krs_param_adv}. We first searched for the hyper-parameters of ERM that achieved the lowest classification error on unperturbed images independent of the training set, averaged across seeds $\{0, 10, 20\}$. The same procedure was then repeated for ARKS, WRM and PGD; the hyper-parameters that most closely recovered the lowest possible error given by ERM were selected.

We note that although WRM~\eqWRM requires for the Lagrangian relaxation coefficient $2y \geq L$, the Lipschitz constant of the loss gradient,  $L$ is hard to compute for our models. In order to provide a fair comparison to \sdro, we searched for the optimal $y$ within a large range. Similarly, the PGD benchmark was tuned to defend against worst-case disturbances $\delta$ within a box $\{\delta: \| \delta\|_{\infty} \leq \Delta  \}$. $\Delta$ was set to 0.3, the maximum magnitude of the disturbances considered during evaluation. For both the optimization of model weights and the inner optimization of worst-case perturbations, we consider common optimizers such as stochastic gradient descent (SGD), Adam \citep{kingma2014adam} and AMSGrad \citep{Tran_2019}. 

\begin{table}[h]
  \renewcommand{\arraystretch}{1.3}
  \begin{center}
      \begin{tabular}[h]{@{\hspace{0.15cm}}l@{\hspace{0.5cm}} l@{\hspace{0.01cm}} c@{\hspace{0.22cm}} c@{\hspace{0.2cm}} c@{\hspace{0.15cm}}}
          \parbox{1.1cm}{Algorithm} & Hyper-parameters & Fashion-MNIST & CIFAR-10 & CelebA \\  \hline 
          \multirow{ 5}*{\parbox{1.1cm}{ALL}} & training epochs & 45 & 65  & 45 \\
          &batch size $\in \{64, 128, 256\}$ & 256 & 128 & 128 \\
          &optimizer $\in \{$SGD,  AMSGrad, Adam$\}$ & AMSGrad & SGD & AMSGrad \\
          & learning rate $\in \left[0.0001, 0.2\right]$ & 0.001 & 0.1 & 0.001 \\
          &decay learning rate $\in \{$True, False$\}$ & False & True & False \\ \hline
          \multirow{ 3}*{\parbox{1.1cm}{\sdro, WRM, PGD}}&inner optimization epochs & 15 & 15 & 15 \\
          &inner optimizer  $\in \{$SGD,  AMSGrad, Adam$\}$ & AMSGrad & AMSGrad & AMSGrad\\
          &decay inner learning rate $\in \{$True, False$\}$ & False & False & False\\\hline
          \multirow{ 2}*{\parbox{1.1cm}{\sdro}}&inner learning rate $\in \left[0.0001, 0.1\right]$ & 0.01 & 0.001 & 0.002\\
          &kernel bandwidth $\sigma \in \left[0.01, 10\right]$  & 0.5& 0.1& 0.2 \\ \hline
          \multirow{ 2}*{\parbox{1.1cm}{WRM}}&inner learning rate $\in \left[0.0001, 0.1\right]$ & 0.05& 0.001& 0.002\\
          &Lagrangian coefficient $y \in \left[0.01, 1000\right]$  & 1.0 & 50 & 4.0\\ \hline
          PGD &inner learning rate $\in \left[0.0001, 0.1\right]$ & 0.001 & 0.0001 & 0.0005\\
          \bottomrule       
      \end{tabular} 
  \end{center}
\caption{Hyper-parameter configuration for classification tasks on the Fashion-MNIST, CIFAR-10 and CelebA datasets. For clarity, we indicate a range for large sets of hyper-parameter values.}
\label{tab:krs_param_adv}
\end{table}

In the CIFAR-10 classification task, the learning rate is decayed by a multiplicative factor of 0.1 at steps 25 and 50. Additionally, to prevent ERM from overfitting and ensure a fair comparison across algorithms, they all use a weight decay of $0.0001$ as per the authors of the ResNet architecture \citep{he2015deep}. 

\paragraph{Adversarial attacks.}

In our evaluation, we perturb test images with worst-case disturbances $\delta$ within a box $\{\delta: \| \delta\|_{\infty} \leq \Delta  \}$. We consider two types of adversarial attacks using this norm. Firstly,  \textit{black-box attacks} for which the disturbances are generated by attacking the model trained with ERM (for each random seed) using PGD and FGSM. Secondly, instead of evaluating each objective on ERM-adversarial loss, we perform \textit{white-box attacks} using PGD on each model individually. In all evaluations, the perturbed images $x + \delta$ are clipped to the valid image range $\left[0, 1 \right]$. PGD performs 15 iterations of gradient ascent on $\delta$ with a learning rate $\alpha=0.03$ for Fashion-MNIST, and $\alpha=0.02$ for CIFAR-10 and CelebA. FGSM performs one iteration of gradient ascent by design.

\paragraph{Supplementary results on black-box attacks.} 
We repeat Figure \ref{fig:pgd_attacks} with the PGD adversarial training algorithm and for FGSM attacks with respect to $\|.\|_{\infty}$. The results for Fashion-MNIST are shown on Figure \ref{fig:fmnist_with_pgd}, and for CIFAR-10 on Figure \ref{fig:cifar_with_pgd}. The same procedure is also applied to binary classifiers trained on CelebA face images, with the results shown on Figure \ref{fig:celeba_attacks}. 

Across all tests, we see that ERM offers the least robustness. This is expected for an optimistic statistical estimator that underestimates risk and is a well-known fact in stochastic optimization~\citep{shapiroLecturesStochasticProgramming2014}.
We emphasize that we included the comparison with the PGD benchmark for completeness. In reality, \sdro is only directly comparable with WRM since they are DRO approaches while PGD is based on RO, as we have discussed in the main text.
We do not intend to show PGD to be less robust than ARKS and WRM since the robustness of DRO and RO depends on the choices of uncertainty and ambiguity sets.

ARKS and WRM exhibit similar adversarial profiles, with ARKS offering slightly more robustness as the magnitude of the adversarial perturbations increase. We use the hyper-parameter values outlined in Table~\ref{tab:krs_param_adv}, but also include \sdro with a higher $\sigma$ and WRM with a lower $y$ (but otherwise optimal hyper-parameter values), exhibiting improved robustness for a small sacrifice on classifying unperturbed images. Further increasing $\sigma$ or decreasing $y$ would increase this test-time penalty. However WRM would rapidly become unstable. To use WRM for deep networks such as ResNet, $y$ needs to be tuned to a high value in order to prevent instabilities from the propagation of the input perturbation through the network. 

\begin{figure}[h!]
	\centering
	\begin{subfigure}{0.49\textwidth}
        \centering
        \includegraphics[width=2.7in,valign=c]{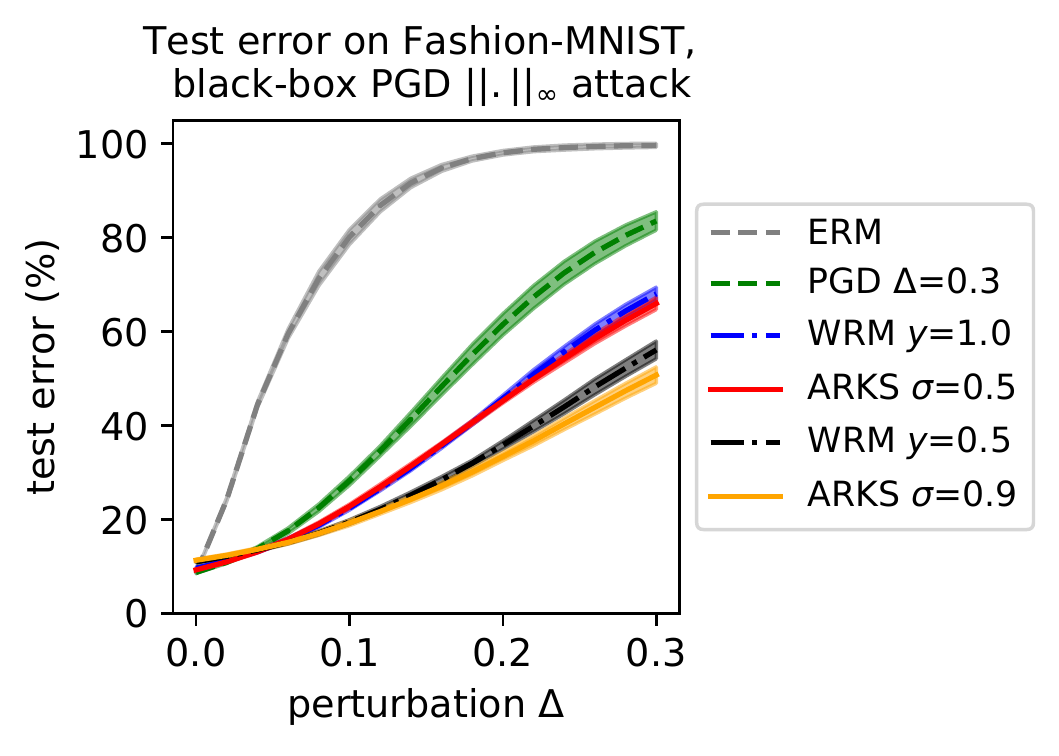}
    \end{subfigure}
    	\begin{subfigure}{0.49\textwidth}
        \centering
        \includegraphics[width=2.7in,valign=c]{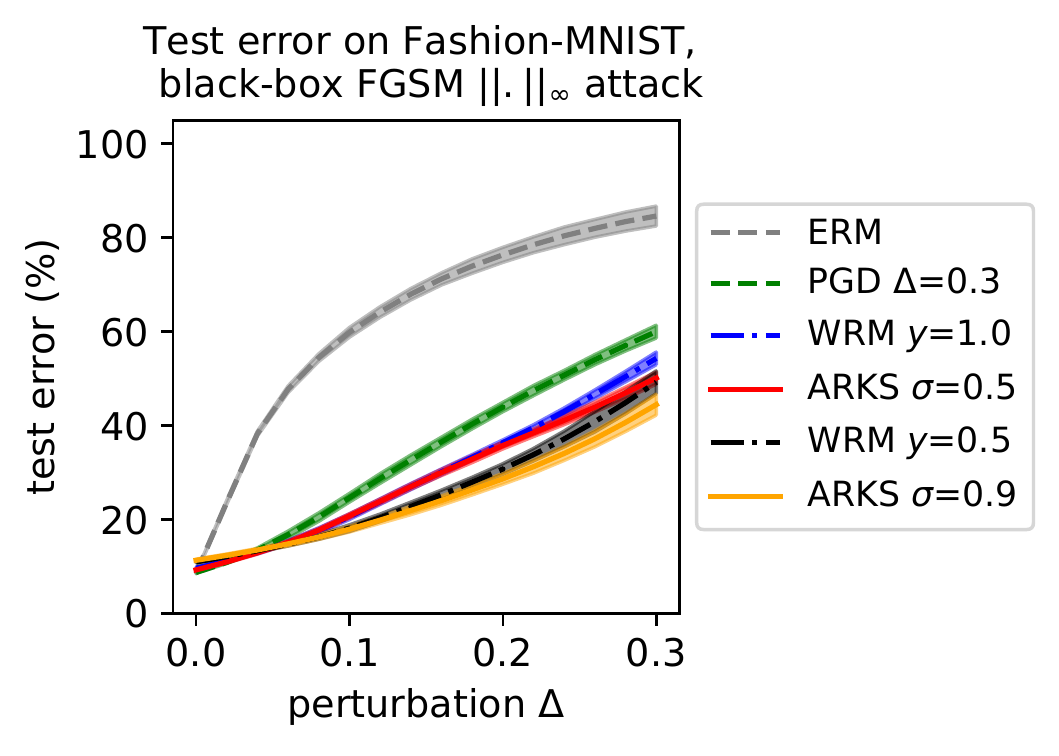}
    \end{subfigure}
    \caption{Black-box PGD (\textbf{left}) and FGSM (\textbf{right}) attacks with respect to $\| . \|_\infty$ on the Fashion-MNIST dataset. We show the classification error on perturbed test images versus the allowed magnitude of the adversarial perturbation  $\Delta$. For all algorithms, we report the mean and standard deviation across 10 random seeds.}
    \label{fig:fmnist_with_pgd}
\end{figure}

\begin{figure}[h!]
	\centering
	\begin{subfigure}{0.49\textwidth}
        \centering
        \includegraphics[width=2.7in,valign=c]{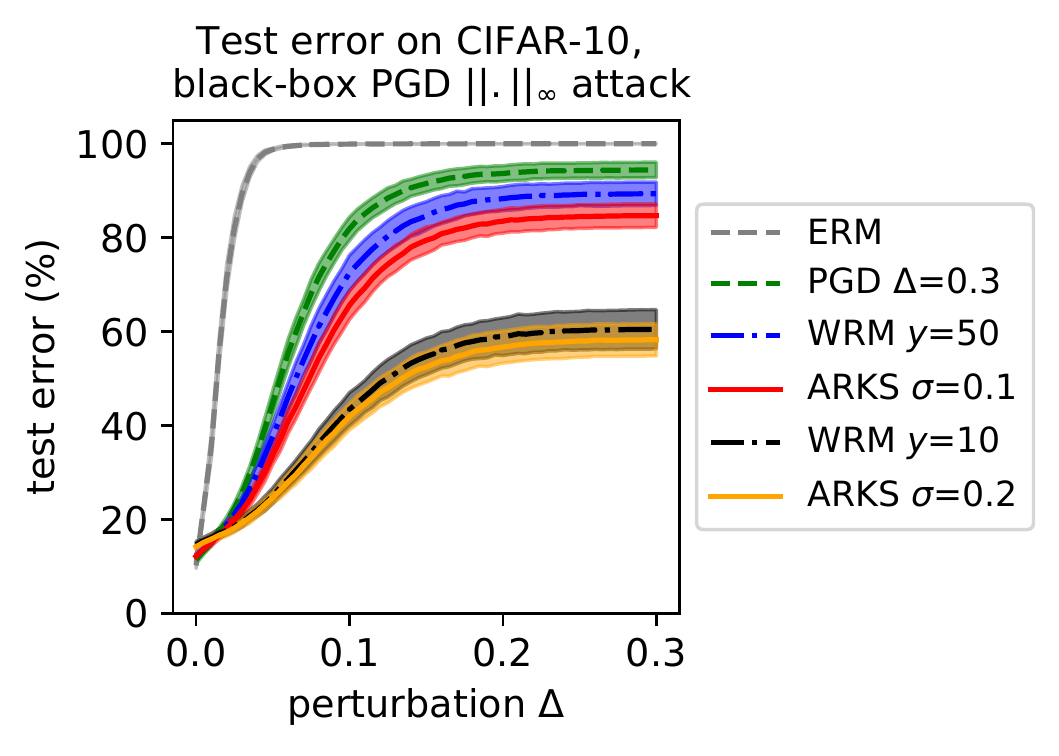}
    \end{subfigure}
    	\begin{subfigure}{0.49\textwidth}
        \centering
        \includegraphics[width=2.7in,valign=c]{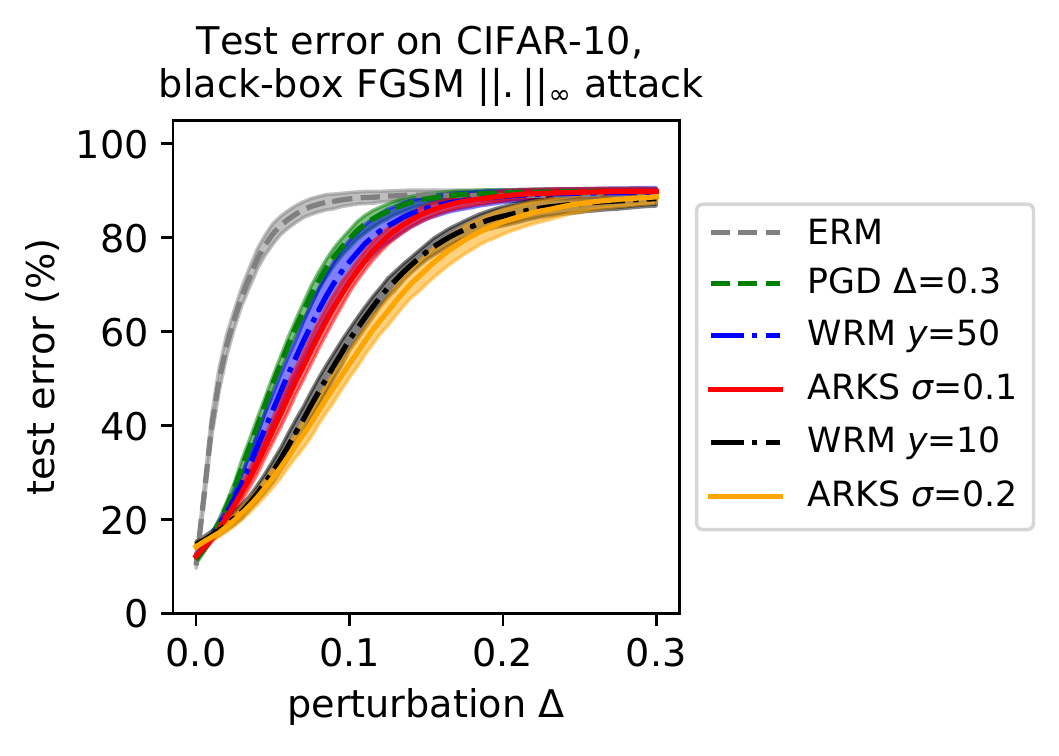}
    \end{subfigure}
    \caption{Black-box PGD (\textbf{left}) and FGSM (\textbf{right}) attacks with respect to $\| . \|_\infty$ on the CIFAR-10 dataset. We show the classification error on perturbed test images versus the allowed magnitude of the adversarial perturbation  $\Delta$. For all algorithms, we report the mean and standard deviation across 10 random seeds.}
    \label{fig:cifar_with_pgd}
\end{figure}

\begin{figure}[h!]
	\centering
	\begin{subfigure}{0.49\textwidth}
        \centering
        \includegraphics[width=2.7in,valign=c]{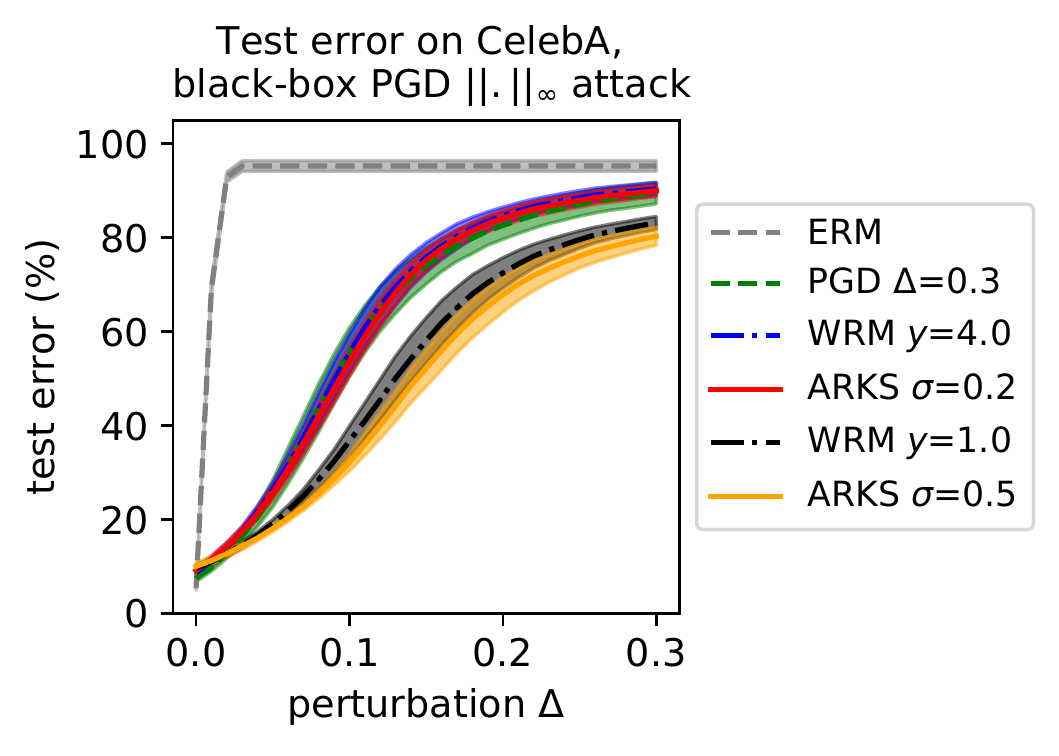}
    \end{subfigure}
    	\begin{subfigure}{0.49\textwidth}
        \centering
        \includegraphics[width=2.7in,valign=c]{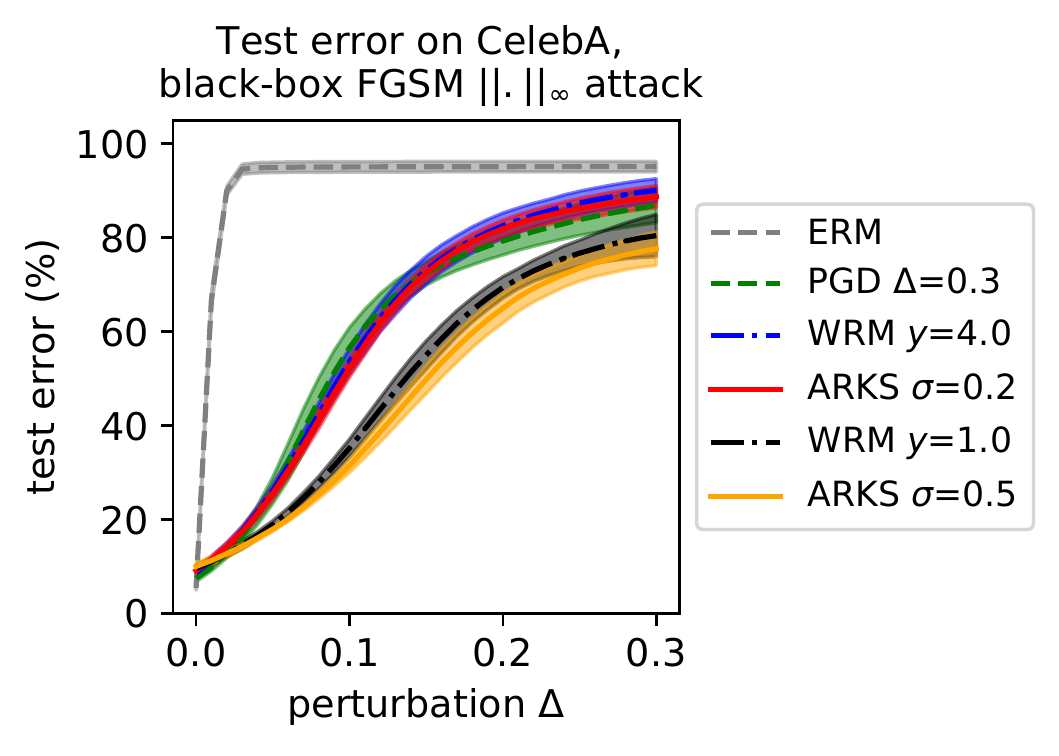}
    \end{subfigure}
    \caption{Black-box PGD (\textbf{left}) and FGSM (\textbf{right}) attacks with respect to $\| . \|_\infty$ on the reduced CelebA dataset. We show the classification error on perturbed test images versus the allowed magnitude of the adversarial perturbation  $\Delta$. For all algorithms, we report the mean and standard deviation across 10 random seeds.}
    \label{fig:celeba_attacks}
\end{figure}

\paragraph{Supplementary results on white-box attacks.} In previous experiments, models trained with each learning objective were evaluated on the same set of perturbed images generated by black-box PGD attacks on the models trained with ERM. In this experiment, each model is evaluated on perturbations generated by white-box PGD $\|\|_\infty$ attacks on itself. The left panel of Figure~\ref{fig:whitebox_pgd_attacks} shows the evaluation results for Fashion-MNIST, and the right panel for CIFAR-10.  

\begin{figure}[h!]
	\centering
	\begin{subfigure}{0.49\textwidth}
        \centering
        \includegraphics[width=2.7in,valign=c]{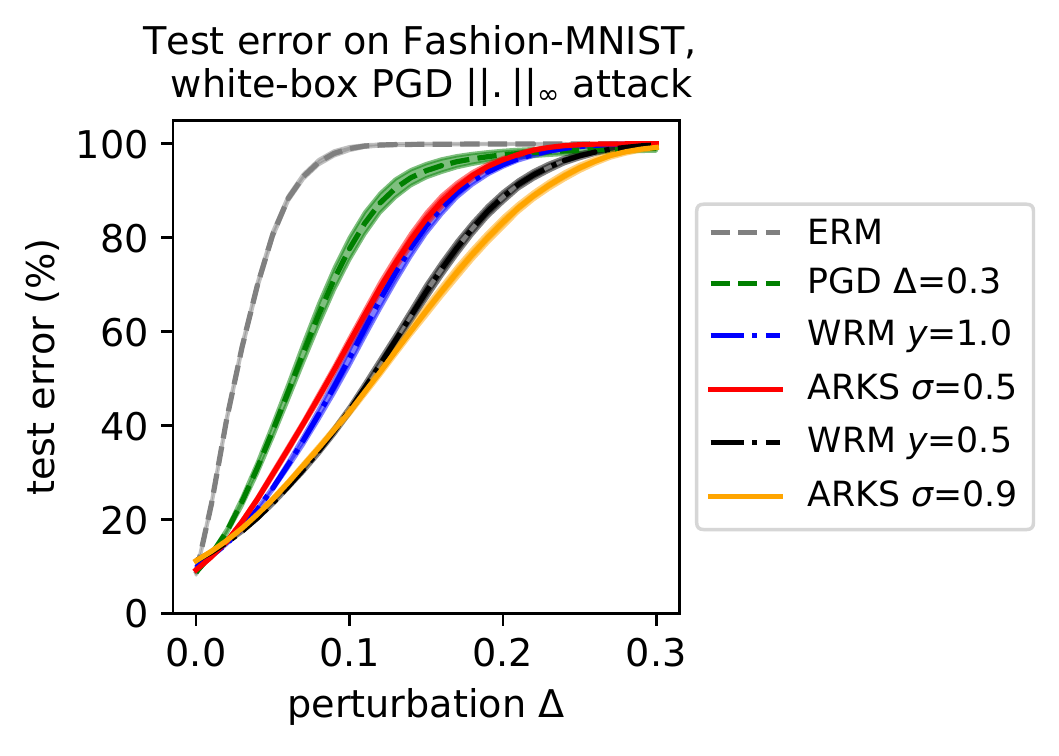}
    \end{subfigure}
    	\begin{subfigure}{0.49\textwidth}
        \centering
        \includegraphics[width=2.7in,valign=c]{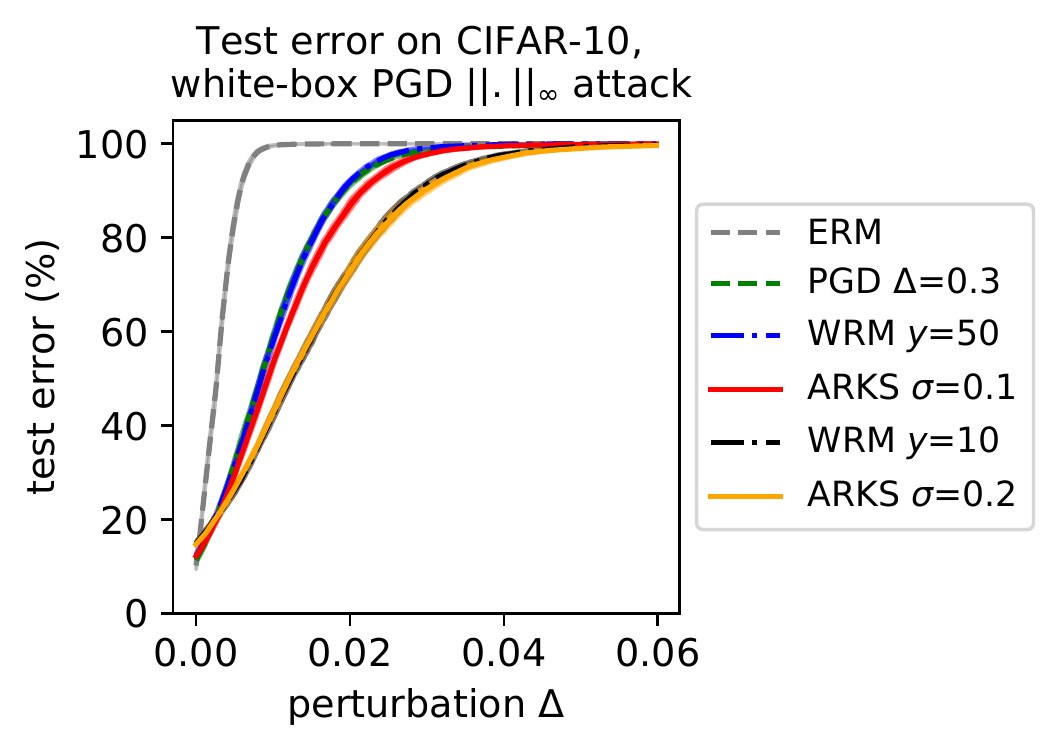}
    \end{subfigure}
    \caption{White-box PGD attack with respect to $\| . \|_\infty$ on the Fashion-MNIST (\textbf{left}) and CIFAR-10 (\textbf{right}) datasets. We show the classification error on perturbed test images versus the allowed magnitude of the adversarial perturbation  $\Delta$. For all algorithms, we report the mean and standard deviation across 10 random seeds. On the right plot, the curves for WRM $\gamma=50$ and PGD almost coincide.}
    \label{fig:whitebox_pgd_attacks}
\end{figure}

\subsection{Comparison with the work of \cite{madryDeepLearningModels2019}}

In this section, we briefly contrast ARKS against the work of \cite{madryDeepLearningModels2019} that introduces PGD, a common adversarial learning algorithm. However, we emphasize that ARKS is best compared to WRM, a state-of-the-art method based on DRO, as PGD is based on RO. We refer to \cite{sinhaCertifyingDistributionalRobustness2017} for extensive comparisons of PGD with WRM.  

Theoretically, as ARKS and WRM are derived using the strong duality of DRO, they are less conservative than RO and therefore PGD. This is reflected in robustifying against \textit{imperceptible attacks} in the evaluations of \cite{sinhaCertifyingDistributionalRobustness2017}. Computationally, the inner maximization of PGD is typically difficult in the case of non-convex losses, while ARKS and WRM both apply convexification (see our discussion in Section \ref{sec:bg_dro}) and smoothing. For instance, the projected gradient step is prone to get stuck in local optima.

Furthermore, ARKS makes contributions in areas that substantially differ from the main idea of \cite{madryDeepLearningModels2019}. Notably, ARKS sheds light on using kernel smoothing theory to robustify deep neural networks. Kernel methods are conventionally not scalable, but our experiments show that ARKS can scale in practical adversarial learning tasks. While previous attempts in robust kernel density estimation exist, the underlying robust kernel methods are not adversarial to specific loss functions nor scalable. ARKS also opens up a new lane of designing kernels for robustness and causal inference. The current work only tested with the default Gaussian kernel, whose empirical performance is already competitive. Finally, our work develops new functional analysis theory for robustness, which does not exist in \citep{madryDeepLearningModels2019}. 

\subsection{Further analysis of \sdro}
In the toy problem for \figtoyA{}, we followed the set-up
from \citep{zhuKernelDistributionallyRobust2020d}
for the robust least-squares example, which appeared in \citep{elghaouiRobustSolutionsLeastSquares1997,boydConvexOptimization2004}.
We formulate the optimization problem
$\min_{ \var}{\|A(\xi)\cdot \var-b\|_2^2}$,
where $A(\xi)$
is assumed to be uncertain and given by
$A(\xi) = A_0 + \xi A_1$,
where
$-1\leq\xi\leq1$ is an uncertain variable.
We refer to \citep{zhuKernelDistributionallyRobust2020d} for more details.

The ERM solution to the robust least-squares problem is computed by solving a convex program.
The RO solution is obtained by solving the SDP reformulation in \citep{elghaouiRobustSolutionsLeastSquares1997}. Additional visual insights, such as a comparison of our approach with the ERM and RO solution, a comparison of empirical and adversarial distribution, are highlighted in Figure~\ref{fig:tradeoff_advkde}. We refer to the caption for more details.

For \sdro, we solved the program~\eqref{eq:ksmooth} using stochastic gradient descent ascent (GDA): in each iteration, we sample a mini-batch $\{\xi_i\}$, then performed gradient ascent to maximize the inner objective of \eqref{eq:ksmooth} w.r.t. the inner variable $u$.
In practice, we performed $10$ steps inner gradient ascent using the L-BFGS routine. 
The inner maximization problem is illustrated in Figure~\ref{fig:sgdstep_landscape}.
Following that, we took one outer gradient descent step w.r.t. the decision variable \var, and repeated the loop. See Algorithm \ref{alg:smooth} for more details.
The outer optimization problem is solved using the L-BFGS optimization routine in PyTorch, with a learning rate of 1.
 \begin{figure}[h]
    \centering
    \begin{subfigure}{0.49\textwidth}
        \centering
        \includegraphics[height=1.7in,valign=c]{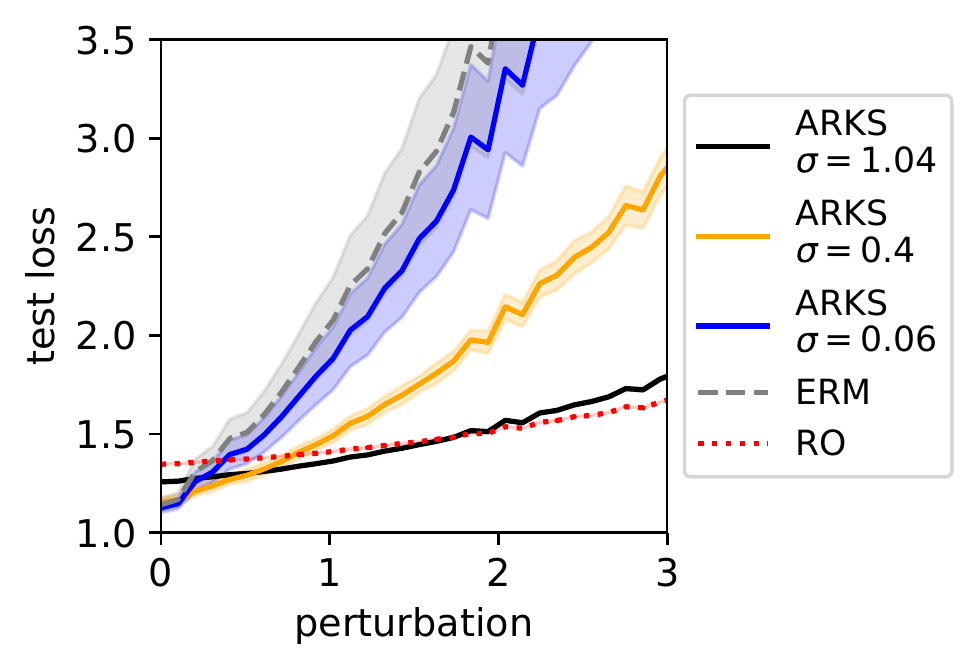}
        \label{fig:tradeoff}
    \end{subfigure}
    \begin{subfigure}{0.49\textwidth}
        \centering
        \includegraphics[height=1.7in,valign=c]{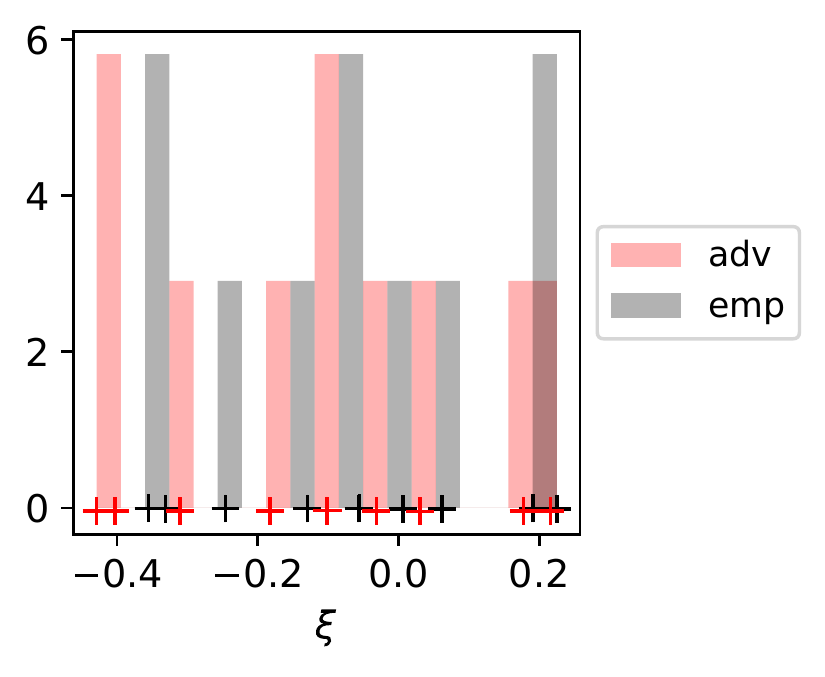}
        \label{fig:interpl}
    \end{subfigure}
    \caption{
    (\textbf{left})
        We plot the performance-robustness trade-off of \sdro for various width settings (black, yellow, blue).
        We create settings of perturbed test distribution (different from the training data distribution, with the random variables satisfying
        $X_\textrm{test} = (1+\delta)\cdot X_\textrm{train}$), with increasing amounts of distribution shift parameter $\delta$.
        We compare with ERM and the worst-case robust optimization (RO) solution of~\citep{elghaouiRobustSolutionsLeastSquares1997}.
        We see that \sdro with large width $\sigma$ is more robust and conservative, tending towards RO.
        When width $\sigma$ is small, \sdro achieves better performance but less robust under a large distribution shift.
        Overall, \sdro performs as Proposition~\ref{thm:recover_ro} indicates, achieving a balance of moderate performance and robustness between ERM and RO.
        For every algorithm, we ran train $10$ independent models.
        The error bars are in standard errors.
    (\textbf{right})
        Histogram density estimation with $\sigma=0.41$ (as used in \sdro) for both the empirical data (black) and the perturbed (adversarial) points (red). The closed-form MMD estimator~\citep{grettonKernelTwoSampleTest2012} between the samples and the adversarial samples evaluates to MMD $= 0.167 \pm 0.02$, averaged over 10 independent runs.
    }
    \label{fig:tradeoff_advkde}
    \end{figure}

\subsection{Results for additional models and data sets}
\label{sec:neural}
In addition to the linear model in the RLS example and the previously reported benchmarks on CIFAR-10, Fashion-MNIST, and CelebA datasets,
we report other results using \sdro with a smaller neural network model.
We used a multi-layer perceptron with two fully connected hidden layers, with 32 hidden units for each layer.
The multi-layer perceptron (MLP) uses the ELU activation because of its smoothness property.
We trained $5$ independent models for every setting and use stochastic weight averaging \citep{izmailov2018averaging} for all neural network training.
We report the results in Figure~\ref{fig:diabetes_samples} and the caption therein. For exact hyperparameter configurations of the MLP training, consult Table~\ref{tab:mlp_param}.
\begin{table}[h]
  \renewcommand{\arraystretch}{1.3}
  \begin{center}
      \begin{tabular}[h]{l@{\hspace{0.5cm}} l@{\hspace{0.5cm}} c@{\hspace{0.5cm}} c@{\hspace{0.5cm}} c}
          \parbox{1.1cm}{Algorithm} & Hyper-parameters & Diabetes & Iris \\  \hline 
          \multirow{ 5}*{\parbox{1.1cm}{ERM \& \sdro}}
          & batch size & 256 & 128\\
          & optimizer & Adam & SGD \\
          & learning rate & 0.001 & 0.1\\
          & epochs & 2000 & 2000 \\
          \hline
          \multirow{ 3}*{\parbox{1.1cm}{\sdro}}
          & inner optimizer & L-BFGS & L-BFGS \\
          & inner learning rate & 1 & 1 \\
          & inner epochs & 10 & 10 \\
          \bottomrule
      \end{tabular}
  \end{center}
  \caption{Hyperparameter configurations for experiments using a multi-layer perceptron as model}
  \label{tab:mlp_param}
\end{table}

 We report results on the diabetes regression dataset \footnote{\label{footnote1}available at \url{https://scikit-learn.org/stable/datasets/toy_dataset.html}} and the iris plants classification dataset \cref{footnote1}. %
To test the robustness property of the methods, we add the perturbation to the test data samples using the following rule
\begin{equation*}
  X_\textrm{perturbed} = X_\textrm{test} + d \cdot \textrm{Uniform}(-1, 1).
\end{equation*}
We increase the perturbation magnitude $d$ from $0$ to $1$. The results are reported in Figure~\ref{fig:diabetes_samples}.
\begin{figure}[h]
	\centering
  \includegraphics[width=2.6in,valign=c]{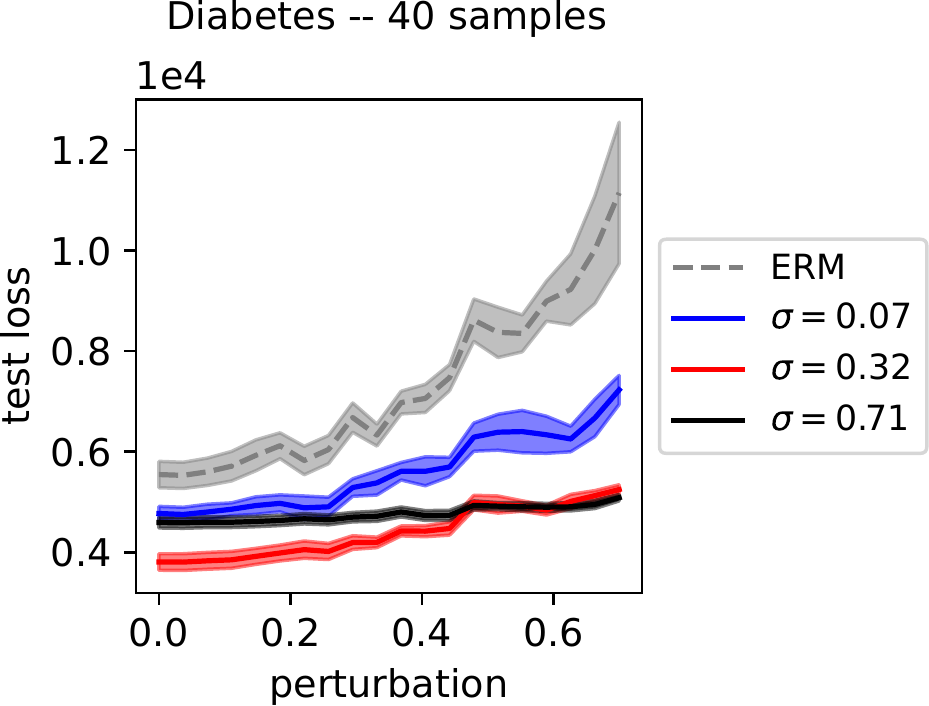}
  \includegraphics[width=2.6in,valign=c]{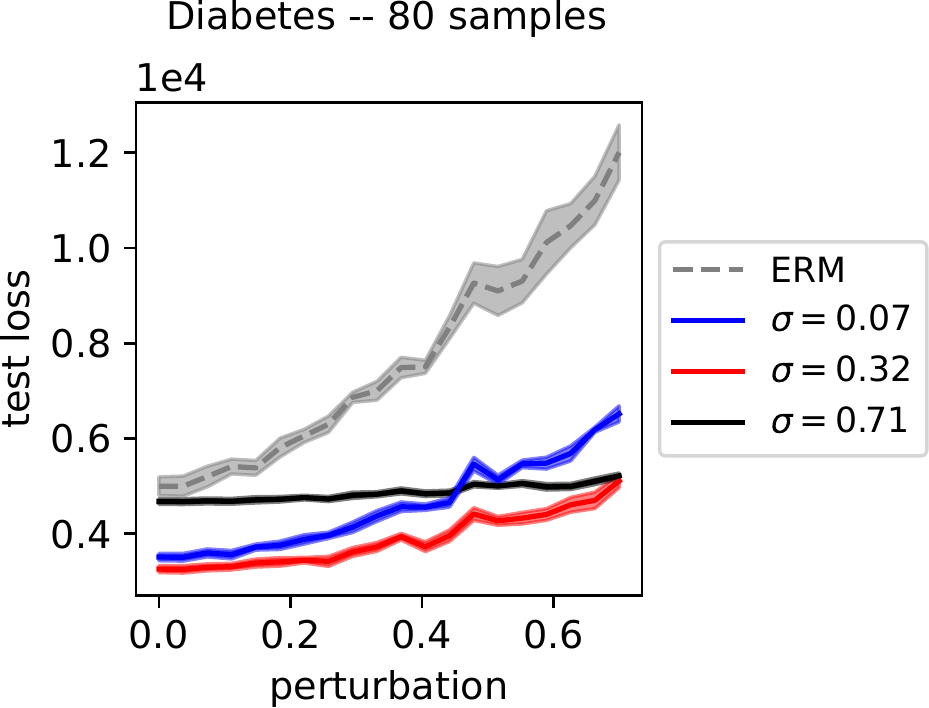}
   
   \vspace{0.7cm}
   
  \includegraphics[width=2.6in,valign=c]{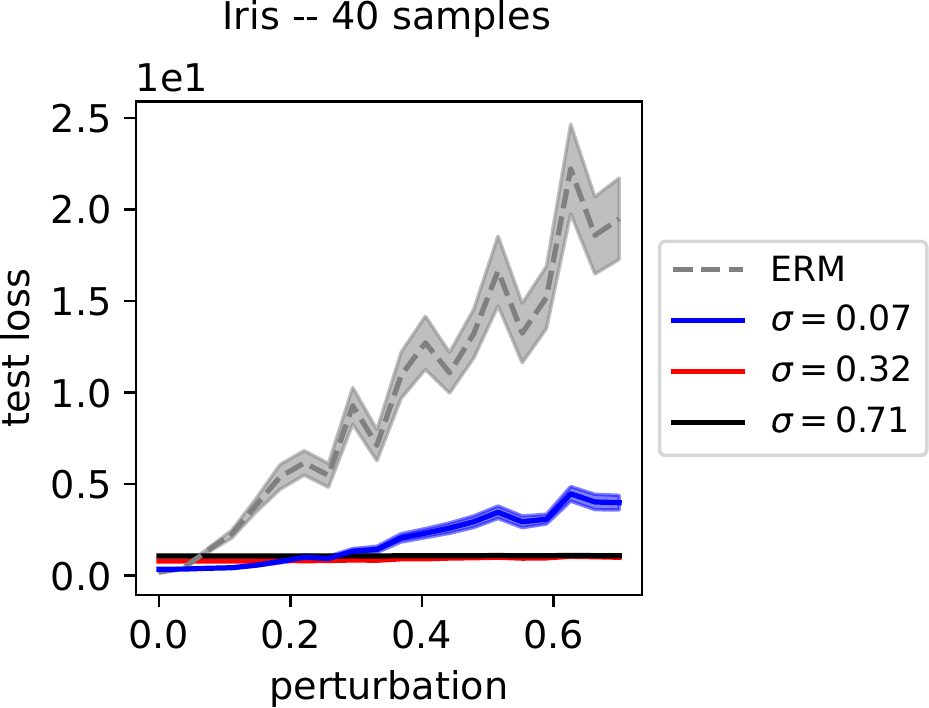}
  \includegraphics[width=2.6in,valign=c]{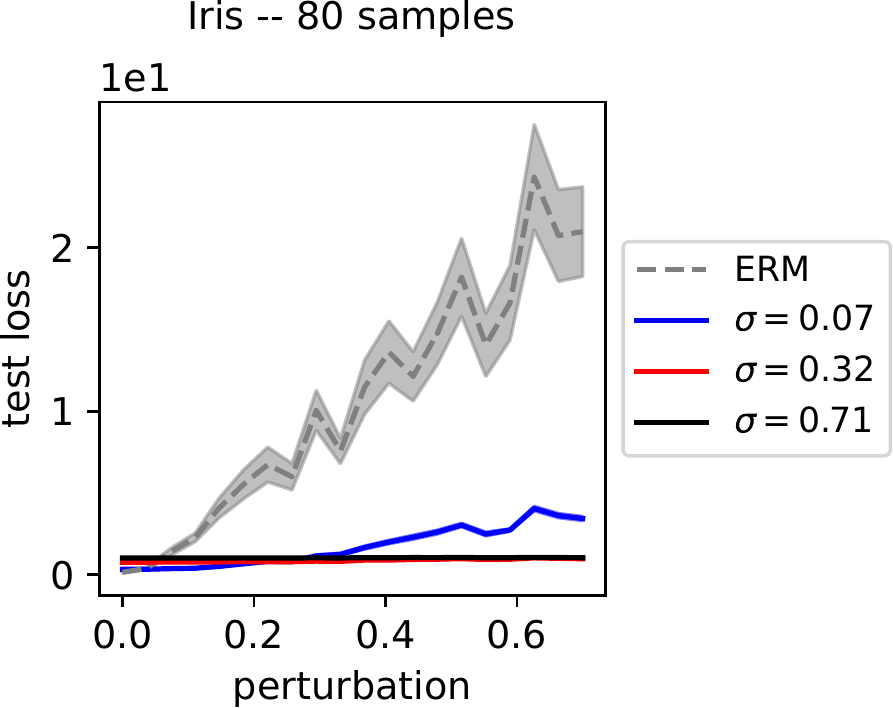}
  \caption{
  We trained the neural network model with the \sdro algorithm.
  We compare the results with the ERM solutions.
  The top-left figure shows model evaluations trained with $40$ samples from the diabetes dataset; the top-right figure corresponds to $80$ training samples from the diabetes dataset.
  The bottom-left figure shows model evaluations trained with $40$ samples from the iris plant dataset and the bottom-right figure for $80$ training samples from the same dataset.
  Across all figures,
  we observe that the ERM performance degrades as the perturbation of the test data increases.
  By contrast, and as expected, \sdro has better robustness against the distribution shift.
  For smaller kernel width $\sigma$, the curve approaches the ERM solution.
  With increasing kernel widths, the \sdro solution becomes more robust but is also more conservative.
  Note that the curves are the mean test errors; the error bars denote the standard errors}
  \label{fig:diabetes_samples}
\end{figure}

\end{document}